\documentclass{article}

\usepackage[main, final]{neurips_2025}

\usepackage[utf8]{inputenc} 
\usepackage[T1]{fontenc}

\usepackage{amsmath, amssymb, amsthm, dsfont}
\usepackage{graphicx}
\usepackage[export]{adjustbox}
\usepackage{booktabs, multirow, makecell, diagbox}
\usepackage{enumitem}
\usepackage{tikz}
\usepackage{tcolorbox}
\usepackage{natbib}
\usepackage{url}
\usepackage{xcolor}
\usepackage{nicefrac}
\usepackage{microtype}
\usepackage{hyperref}
\usepackage{algorithm}
\usepackage{algpseudocode}

\newtheorem{theorem}{Theorem}[section]
\newtheorem{lemma}[theorem]{Lemma}   
\theoremstyle{remark}

\newtheorem{corollary}{Corollary}

\DeclareMathOperator{\tr}{tr}

\DeclareMathOperator{\E}{\mathbb{E}}

\usepackage{bbding, pifont, wasysym}
\title{Diversity Is All You Need for Contrastive Learning: Spectral Bounds on Gradient Magnitudes}

\author{
  Peter Ochieng\thanks{Accepted as a Poster at the 39th Conference on Neural Information Processing Systems (NeurIPS 2025)} \\
  Department of Computer Science \\
  University of Cambridge \\
  \texttt{po304@cam.ac.uk}
}

\begin{document}

\maketitle

\begin{abstract}
We derive non-asymptotic spectral bands that bound the squared InfoNCE gradient norm via alignment, temperature, and batch spectrum, recovering the \(1/\tau^{2}\) law and closely tracking batch-mean gradients on synthetic data and ImageNet. Using effective rank \(R_{\mathrm{eff}}\) as an anisotropy proxy, we design spectrum-aware batch selection, including a fast greedy builder. On ImageNet-100, Greedy-64 cuts time-to-67.5\% top-1 by 15\% vs.\ random (24\% vs.\ Pool--P3) at equal accuracy; CIFAR-10 shows similar gains. In-batch whitening promotes isotropy and reduces 50-step gradient variance by \(1.37\times\), matching our theoretical upper bound.
\end{abstract}

\section{Introduction}
\label{sec:intro}

Contrastive learning is a \emph{label-efficient} paradigm for representation learning: it
pulls together \emph{positive} views of the same instance while pushing apart \emph{negatives}
in the mini-batch, uncovering latent structure without class labels. Yet this tug-of-war is delicate:
weak positives stall learning, whereas overly similar negatives drive gradients toward zero, even in
frameworks that rarely collapse in practice such as
SimCLR~\citep{chen2020simple} and MoCo~\citep{he2020momentum}.
Most prior work tackles this at the \emph{pair level} (e.g., hard-negative mining and debiased losses),
but overlooks a batch-level signal: the \emph{spectrum} of the embedding cloud.
If that cloud is too \emph{narrow} (high anisotropy; low effective rank), negatives become redundant; if too
\emph{wide} (near-isotropic with uniformly small pairwise similarities), the softmax distribution flattens.
Empirically, we find training is fastest inside a \emph{moderate-diversity window}—rich enough to avoid collapse,
yet narrow enough to preserve directional signal.

We formalise and exploit this observation via three contributions:

\begin{enumerate}
\item \textbf{Gradient-norm spectral band.}
      We derive sharp, non-asymptotic bounds on the \emph{squared} InfoNCE gradient norm,
      decomposing contributions from positive alignment, finite-sample variance, and batch anisotropy.
      The band depends only on in-batch (or queue) second moments, making the analysis—and the
      resulting diagnostics—agnostic to whether negatives are drawn in-batch (SimCLR) or from a queue (MoCo).

\item \textbf{Diversity-aware batch construction.}
      We propose two lightweight samplers that keep training inside the diversity window:
      (i) a \emph{pool} selector that targets high effective rank \(R_{\mathrm{eff}}\);
      and (ii) a streaming \emph{Greedy-\(m\)} builder that adds the sample with the largest
      spectral-diversity gain.
      In our measurements, Greedy-\(m\) adds \(\lesssim 1\%\) on-GPU overhead, while large host-side pools can
      incur \(\sim 8\,\mathrm{ms}\)/iteration on a V100 (see \S\ref{sec:runtime}).

\item \textbf{Experiments from 100 to 1000 classes.}
      On ImageNet-100 our samplers reduce wall-clock time to \(67.5\%\) top-1 by \(\sim 15\%\) with no accuracy loss.
      On ImageNet-1k, linear-eval accuracy increases with \(R_{\mathrm{eff}}\) and plateaus once
      \(R_{\mathrm{eff}}/C \gtrsim 0.9\), consistent with the RankMe view that diversity gains saturate near this
      threshold (\S\ref{app:rankme}).
\end{enumerate}

\subsection{Notation}
\label{sec:notation}
Let batch size $n>2$ and embedding dimension $d$. Each embedding
$z_i\in\mathbb{S}^{d-1}\subset\mathbb{R}^d$ has a designated positive $z_{i^+}$.
Write the set of \emph{others} as
\(
\mathcal{S}_i:=\{1,\dots,n\}\setminus\{i\},\qquad |\mathcal{S}_i|=n-1,
\)
and the \emph{negatives} as
\(
\mathcal{N}_i^-:=\mathcal{S}_i\setminus\{i^+\},\qquad |\mathcal{N}_i^-|=n-2.
\)
Let $Z=[\,z_1;\dots;z_n\,]\in\mathbb{R}^{n\times d}$ collect rows $z_i^\top$, and define pairwise
similarities $s_{ij}=z_i^\top z_j\in[-1,1]$ with softmax temperature $\tau>0$.

\paragraph{Second moments.}
Define the batch second moment and the per-anchor “others”/“negatives-only” moments:
\(
\hat{\Sigma}:=\frac{1}{n}\sum_{j=1}^n z_j z_j^\top,\quad
\tilde{\Sigma}_i:=\frac{1}{n-1}\sum_{j\in\mathcal{S}_i} z_j z_j^\top,\quad
\tilde{\Sigma}_i^-:=\frac{1}{n-2}\sum_{j\in\mathcal{N}_i^-} z_j z_j^\top.
\)
Under unit norms, $\tr(\hat{\Sigma})=\tr(\tilde{\Sigma}_i)=\tr(\tilde{\Sigma}_i^-)=1$.
Exact relations:
\(
\tilde{\Sigma}_i=\frac{n}{n-1}\!\left(\hat{\Sigma}-\frac{1}{n}z_i z_i^\top\right),\quad
\tilde{\Sigma}_i^-=\frac{n}{n-2}\!\left(\hat{\Sigma}-\frac{1}{n}(z_i z_i^\top+z_{i^+} z_{i^+}^\top)\right).
\)
Hence the spectral bounds used later:
\[
\lambda_{\max}(\tilde{\Sigma}_i)\le \frac{n}{n-1}\,\lambda_{\max}(\hat{\Sigma}),\quad
\lambda_{\max}(\tilde{\Sigma}_i^-)\le \frac{n}{n-2}\,\lambda_{\max}(\hat{\Sigma}).
\]
We write
\(
\hat\sigma:=\lambda_{\max}(\hat{\Sigma}),\quad
\sigma_i:=\lambda_{\max}(\tilde{\Sigma}_i),\quad
\sigma_*^{(i)}:=\lambda_{\max}(\tilde{\Sigma}_i^-)\in[1/d,\,1]
\)
(minimum at isotropy).

\subsection{Bounding $\|\nabla_{z_i}\mathcal L_i\|^{2}$ — Spectral Gradient Band}
\label{sec:gnsb-overview}

We show that the per-sample InfoNCE gradient is confined to a \emph{spectral band} whose width is governed by alignment, finite-sample noise, batch anisotropy, and temperature.

\paragraph{Setup \& assumptions.}
Let $N^-{=}n{-}2$, $\tilde{\Sigma}_i^-:=\tfrac{1}{N^-}\sum_{j\in\mathcal N_i^-}z_j z_j^\top$, and
$\sigma_*^{(i)}:=\lambda_{\max}(\tilde{\Sigma}_i^-)$.
With unit norms and the Löwner order,
$\tilde{\Sigma}_i^-=\tfrac{1}{N^-}\!\bigl(n\hat\Sigma-z_i z_i^\top-z_{i^+}z_{i^+}^\top\bigr)\preceq\tfrac{n}{n-2}\hat\Sigma
\Rightarrow \sigma_*^{(i)}\le \tfrac{n}{n-2}\hat\sigma$,
where $\hat\Sigma=\tfrac{1}{n}\sum_i z_i z_i^\top$ and $\hat\sigma=\lambda_{\max}(\hat\Sigma)$.
InfoNCE: $\mathcal L_i=-\log p_{ii^+}$ with
$p_{ij}\propto e^{s_{ij}/\tau}$, $s_{ij}=z_i^\top z_j$, $\epsilon_i:=1-p_{ii^+}$, and
$M_i=\sum_{k}p_{ik}z_k$ so $\nabla_{z_i}\mathcal L_i=\tau^{-1}(M_i-z_{i^+})$.
Expectations are over the mini-batch/augmentations.
Assumptions: (A1) unit norms; (A2) negatives i.i.d.\ within the batch and independent of $z_i$ (no assumption on the positive pair).

\begin{theorem}[Gradient–Norm Spectral Band]
\label{thm:gnsb}
Under (A1)–(A2), for a softmax-smoothness constant $c>0$,
\begin{align}
\mathbb{E}\big[\|\nabla_{z_i}\mathcal{L}_i\|^{2}\big]
&\le \frac{3}{\tau^{2}}\!\Big(\mathbb{E}[\epsilon_i^{2}]
      + \tfrac{\mathbb{E}[(1-p_{ii^+})^2]}{N^{-}}\Big)
   + \frac{3\,\mathbb{E}[(1-p_{ii^+})^2]\,\sigma_*}{\tau^{4}}
   + \frac{3c\,\mathbb{E}[(1-p_{ii^+})^2]\,\sigma_*^{2}}{\tau^{6}}, \label{eq:gnsb-upper}
\\
\mathbb{E}\big[\|\nabla_{z_i}\mathcal{L}_i\|^{2}\big]
&\ge \frac{(1-\rho)^{2}}{\tau^{2}},\qquad
\rho := \mathbb{E}\big[\langle M_i, z_{i^{+}}\rangle\big], \label{eq:gnsb-lower}
\end{align}
where $\sigma_*:=\mathbb{E}[\sigma_*^{(i)}]$ (or use the per-anchor form).
\end{theorem}

\paragraph{Sketch.}
Decompose $\delta_i:=M_i-z_{i^+}=A_i+B_i+C_i$ with
$A_i=-\epsilon_i z_{i^+}$,
$B_i=(1-p_{ii^+})\bar z_i^-$, $\bar z_i^-:=\tfrac{1}{N^-}\sum_{j\in\mathcal N_i^-}z_j$,
and $C_i=(1-p_{ii^+})\sum_{j}(q_{ij}-\tfrac{1}{N^-})z_j$ where $q_{ij}=p_{ij}/(1-p_{ii^+})$.
Then $\|\delta_i\|^2\le 3(\|A_i\|^2+\|B_i\|^2+\|C_i\|^2)$,
$\mathbb E\|A_i\|^2=\mathbb E[\epsilon_i^2]$,
$\mathbb E\|B_i\|^2=\mathbb E[(1-p_{ii^+})^2]/N^-$ by (A2),
and a first-order Taylor of the negatives-only softmax around uniform logits
(App.~\ref{app:softmax-taylor}) gives
$\mathbb E\|C_i^{(1)}\|^2\le \tau^{-2}\mathbb E[(1-p_{ii^+})^2]\sigma_*$ and
$\mathbb E\|C_i^{(2)}\|^2\le c\,\tau^{-4}\mathbb E[(1-p_{ii^+})^2]\sigma_*^2$.
For the lower band, $\|M_i\|\le 1$ implies $\|M_i-z_{i^+}\|^2\ge (1-\langle M_i,z_{i^+}\rangle)^2$; apply Jensen.
Full details in App.~\ref{sec:proof}.

\paragraph{Variants.}
(i) \emph{Bounded spectrum:} if $\sigma_*^{(i)}\le \sigma_{\max}$, replace $\sigma_*$ by $\sigma_{\max}$.
Using $\sigma_*^{(i)}\le \tfrac{n}{n-2}\hat\sigma$ gives a batch-measurable ceiling.
(ii) \emph{High probability:} if $z_j$ are i.i.d.\ sub-Gaussian with $\operatorname{tr}\Sigma=1$,
matrix concentration (App.~\ref{app:matrix-concentration}) yields
$\lambda_{\max}(\hat\Sigma)\lesssim \tfrac{1}{d}+O(\sqrt{\tfrac{\log d}{n}})$ and hence
$\sigma_*^{(i)}\le \tfrac{n}{n-2}\lambda_{\max}(\hat\Sigma)$; substitute into \eqref{eq:gnsb-upper}.

\paragraph{Interpretation.}
Softmax error $\epsilon_i$ and the sampling term $1/N^-$ set a baseline that decays with batch size.
Anisotropy enters via $\sigma_*$ at orders $\tau^{-4}$ and $\tau^{-6}$; pushing toward isotropy (e.g., larger $R_{\mathrm{eff}}$, whitening) tightens the ceiling.
Higher alignment $\rho$ shrinks the floor $(1-\rho)^2/\tau^2$ without implying collapse.
\subsection{Spectral Batch Selection}
\label{sec:spectral-batch-selection}

The upper band in Thm.~\ref{thm:gnsb} depends on the per–anchor negatives-only spectrum
$\sigma_*^{(i)}=\lambda_{\max}(\tilde\Sigma_i^-)$ with
$\tilde\Sigma_i^-=\tfrac{1}{N^-}\sum_{j\in\mathcal N_i^-} z_j z_j^\top$, $N^-=n-2$.
Computing $\sigma_*^{(i)}$ for every anchor is costly, so we use the batch second moment
$\hat\Sigma=\tfrac{1}{n}\sum_{i=1}^n z_i z_i^\top$ (trace one under $\|z_i\|{=}1$) and its top eigenvalue
$\hat\sigma=\lambda_{\max}(\hat\Sigma)$ as a proxy. By Löwner order,
\[
\tilde\Sigma_i^-=\tfrac{1}{N^-}\!\bigl(n\hat\Sigma-z_i z_i^\top-z_{i^+}z_{i^+}^\top\bigr)
\preceq \tfrac{n}{n-2}\,\hat\Sigma
\;\Rightarrow\;
\sigma_*^{(i)} \le \tfrac{n}{n-2}\,\hat\sigma,
\]
so lowering $\hat\sigma$ tightens a uniform ceiling on all $\sigma_*^{(i)}$.

\paragraph{Effective rank.}
Let $R_{\mathrm{eff}}:=1/\operatorname{tr}(\hat\Sigma^2)=\big(\sum_k\lambda_k^2\big)^{-1}$ (eigs $\lambda_k$ of $\hat\Sigma$).
Since $\sum_k\lambda_k=1$ and $\sum_k\lambda_k^2\le \hat\sigma$, we get $\hat\sigma\ge 1/R_{\mathrm{eff}}$:
higher $R_{\mathrm{eff}}$ (more isotropy) $\Rightarrow$ smaller $\hat\sigma$ and a tighter band.

\paragraph{Policies.}
Given a candidate pool $\mathcal P_t=\{Z^{(m)}\}_{m=1}^k$ with ranks $R^{(m)}$:
\emph{P1 (stability)} picks $\arg\max_m R^{(m)}$ to minimize $\hat\sigma$;
\emph{P2 (update-magnitude)} picks $\arg\min_m R^{(m)}$ to allow larger steps (higher variance risk);
\emph{P3 (balanced)} picks $\arg\min_m |R^{(m)}-R_\star|$ with $R_\star$ set by running 10th/90th percentiles
(App.~\ref{app:balanced-p3}). In practice: use P1 early (collapse risk), switch to P3 once $R_{\mathrm{eff}}$ stabilizes,
and deploy P2 sparingly to escape flats.

\subsection{Greedy Element–Wise Spectral Builder}
\label{sec:greedy-spectral}

We assemble a batch incrementally to \emph{decrease} the trace of the squared second moment,
thereby \emph{increasing} $R_{\mathrm{eff}}=1/\operatorname{tr}(\Sigma^2)$. Assume unit–norm rows.

\paragraph{Objective.}
For a partial batch $B$ of size $b$, let
\[
\Sigma_B=\tfrac{1}{b}\sum_{z'\in B} z' z'^\top,\qquad
t_B=\operatorname{tr}(\Sigma_B^2),\qquad
q_B(z)=z^\top \Sigma_B z=\tfrac{1}{b}\sum_{z'\in B}\langle z,z'\rangle^2.
\]
Lemma~\ref{lem:trace-change} (App.~\ref{app:second-moment-trace}) gives the one–step update for a unit–norm candidate $z$:
\[
\operatorname{tr}(\Sigma_{B\cup\{z\}}^2)
=\frac{b^2 t_B + 2b\,q_B(z) + 1}{(b+1)^2},
\quad\Rightarrow\quad
\Delta t:=\operatorname{tr}(\Sigma_{B\cup\{z\}}^2)-t_B
=\frac{2b\,(q_B(z)-t_B) + (1-t_B)}{(b+1)^2}.
\]
Thus smaller $q_B(z)$ yields smaller (more negative) $\Delta t$.

\paragraph{Greedy rule (Greedy–$m$).}
At each step select
\[
z^\star=\arg\min_{z\in\mathcal C} q_B(z),
\]
where $\mathcal C$ is a probe set of size $m$ (from a pool or stream). For unit–norm rows this is equivalent to
maximizing the Frobenius diversity $\Delta_F(z\mid B)=\|\Sigma_B-zz^\top\|_F^2=t_B-2q_B(z)+1$.

\paragraph{Cost.}
With cached dot products against $B$, evaluating $q_B(z)$ costs $O(b)$ per candidate; a probe of size $m$ costs
$O(mb)$ per step (or $O(m b d)$ without caching). We update $t_B$ via Lemma~\ref{lem:trace-change} in $O(1)$ after each selection.
A practical Greedy–$m$ realisation (initial seeds, Gram maintenance, streaming) is given in
Algorithm~\ref{alg:greedy-spectral} (App.~\ref{app:second-moment-trace}).

\section{Experiments}
We evaluate on standard self-supervised benchmarks with SimCLR~\cite{chen2020simple} and MoCo~v2~\cite{he2020moco}.
Unless noted, we follow the original hyperparameters/architectures/augmentations; the \emph{only} change is our
spectrum-aware batch selection (§\ref{sec:spectral-batch-selection}–\ref{sec:greedy-spectral}).
\textbf{Datasets:} ImageNet-100 (main), CIFAR-10 / Oxford Pets (linear probe). \textbf{Backbones:} ResNet-18 (default),
ResNet-50 (variant). \textbf{Batch/Temp:} $n{=}256$, $\tau{=}0.2$ (SimCLR) unless stated. \textbf{Metrics:}
final top-1 and \emph{epochs to threshold} (ImageNet-100: 67.5\%/70\%). Significance via paired $t$-tests (95\% CIs). Our main results use SimCLR; MoCo-v2 ablations and queue robustness are in App.~\ref{app:moco-robustness}.
Policy details (P1–P3), the Greedy–$m$ builder, pool sizing, and overheads are in
App.~\ref{app:balanced-p3}–\ref{app:second-moment-trace}; spectrum metrics and evaluation protocols (e.g., RankMe,
$R_{\mathrm{eff}}$) are in App.~\ref{app:rankme}.

\subsection{Synthetic Gradient--Band Verification}
\label{sec:syn-band}

We validate the analytical upper/lower bounds on $\|\nabla_{z_i}\mathcal L_i\|^2$
(§\ref{sec:gnsb-overview}, Thm.~\ref{thm:gnsb}) on controlled synthetic data, using a
\emph{per-batch plug-in} variant of the band (we replace expectations by measured batch
quantities). Draw $x_i\sim\mathcal N(0,I_d)$ and set
\(
z_i \;:=\; \frac{A x_i}{\|A x_i\|}, \qquad
A := U \Lambda^{1/2},\quad
\Lambda := \operatorname{diag}(\lambda_1,\dots,\lambda_d),\ \ \sum_j \lambda_j=1,
\)
with $U$ orthogonal, so $\|z_i\|=1$. The sample second moment
$\Sigma=\tfrac{1}{n}\sum_i z_i z_i^\top$ satisfies $\operatorname{tr}(\Sigma)=1$ exactly.
In high dimension ($d{=}1024$), the trace-one spectrum empirically concentrates tightly
around $(\lambda_j)$ (median deviation $<1\%$ in our runs).

To control anchor–positive alignment, set
\(
z_i^+ := \rho_{\text{gen}}\, z_i + \sqrt{1-\rho_{\text{gen}}^{\,2}}\, u_\perp,
\)
where $u_\perp$ is a random unit vector orthogonal to $z_i$; then $\|z_i^+\|=1$ and
$\langle z_i,z_i^+\rangle=\rho_{\text{gen}}$.

\paragraph{Quantities and band.}
Compute $s_{ij}:=z_i^\top z_j$, softmax weights $p_{ij}$ over $\{i^+\}\cup\mathcal N_i^-$,
$M_i:=\sum_{j} p_{ij} z_j$, and $g_i=\tau^{-1}(M_i-z_{i^+})$.
We vary
\(
\tau \in \{0.05,0.1,0.2,0.3\},\qquad
\lambda_1 \in \{1/d,\,0.3,\,0.6,\,1.0\},\qquad
\rho_{\text{gen}} = 0.6 + 0.4\,\lambda_1,
\)
to couple higher anisotropy with tighter positives. For each setting we generate
$10{,}000$ batches (batch size $n{=}256$, dimension $d{=}1024$).

For the \emph{upper band}, we use the \emph{negatives-only} moment for each anchor
\[
\tilde\Sigma_i^- := \tfrac{1}{N^-}\!\sum_{j\in\mathcal N_i^-} z_j z_j^\top,\qquad
N^-{=}n{-}2,\qquad
\sigma_*^{(i)}:=\lambda_{\max}(\tilde\Sigma_i^-),
\]
and plug $\sigma_*^{(i)}$ into Thm.~\ref{thm:gnsb}’s spectral terms (per-anchor, no expectation).
For the \emph{lower band}, we compute
\(
\rho_i := \langle M_i, z_{i^+}\rangle,\qquad \bar\rho := \tfrac{1}{n}\sum_i \rho_i,
\)
and use $(1-\bar\rho)^2/\tau^2$ as the plug-in lower bound (§\ref{sec:gnsb-overview}).

\paragraph{Results.}
As shown in Fig.~\ref{fig:s1-synthetic}, across all 16 configurations, at least
$99.9\%$ of measured $\|g_i\|^2$ fall within the predicted band, indicating that
the bounds remain tight and predictive under wide variation in temperature,
alignment, and spectral concentration.

\begin{figure}[H]
  \centering
  \includegraphics[width=0.7\textwidth]{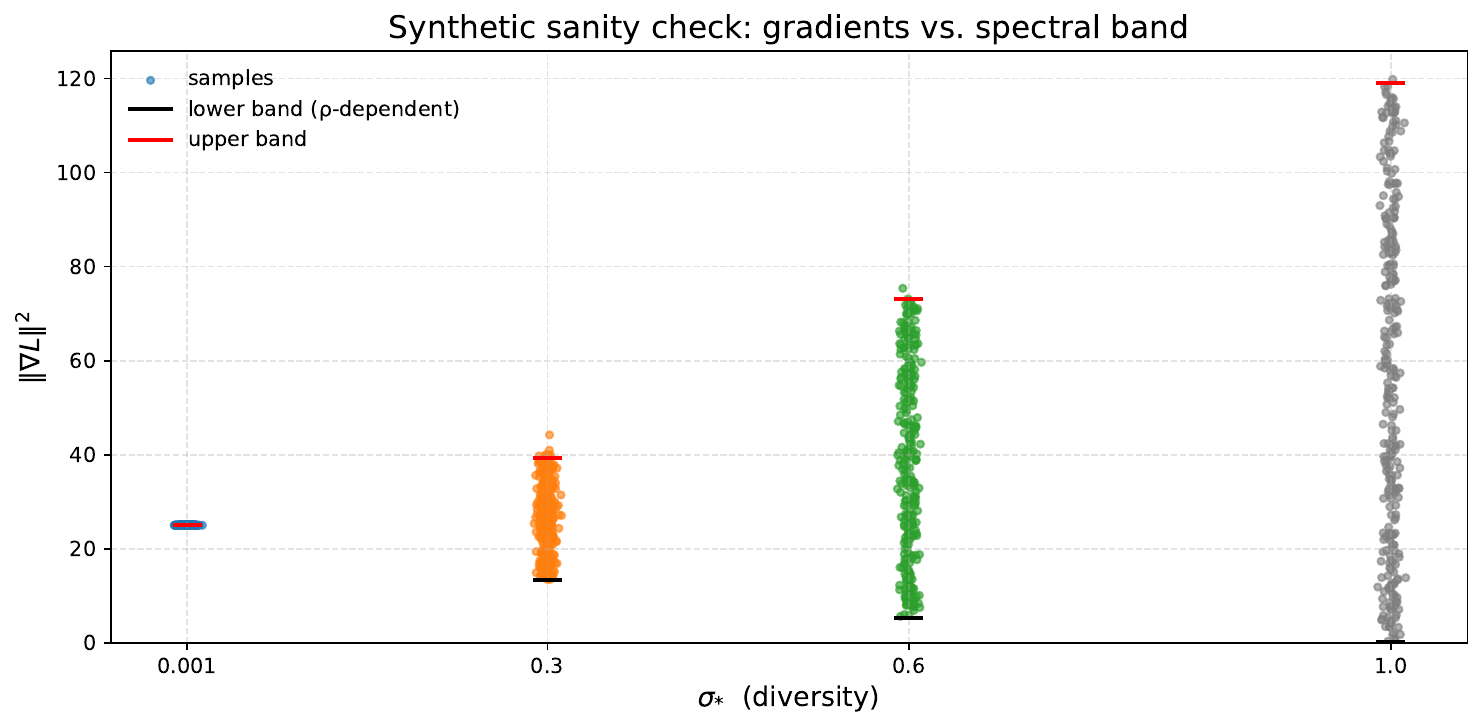}
  \caption{\textbf{Synthetic validation of the gradient--norm spectral band.}
  Measured gradients $g_i$ (blue) vs.\ plug-in lower (black) and upper (red) bounds.
  Example shown for $\tau=0.1$; containment ($\ge 99.9\%$) held for all 16 settings.}
  \label{fig:s1-synthetic}
\end{figure}

\subsection{Gradient–Norm Scaling with Temperature}
\label{sec:s1-temp-scaling}

With the batch \emph{geometry} held fixed (all cosines, including a fixed anchor–positive cosine), the
gradient–norm spectral band (Thm.~\ref{thm:gnsb}, §\ref{sec:gnsb-overview}) predicts a leading
\(1/\tau^{2}\) dependence for the squared gradient, up to higher–order corrections:
\begin{equation}
\label{eq:temp-prediction}
\mathbb{E}\!\bigl[\|\nabla_{z_i}\mathcal L_i\|^{2}\bigr]
= \frac{3}{\tau^{2}}\!\left(
\mathbb{E}[\epsilon_i^{2}]
+ \frac{\mathbb{E}[(1-p_{ii^+})^{2}]}{N^{-}}
\right) \;+\; O(\tau^{-4}) \;+\; O(\tau^{-6}),
\end{equation}
where the prefactor depends on the softmax weights \(p\) through \(\epsilon_i=1-p_{ii^{+}}\).
When relative logit margins remain stable across the sweep, this yields the expected \(1/\tau^{2}\) scaling.

\paragraph{Setup.}
We use the synthetic construction of §\ref{sec:syn-band} with batch size \(n{=}256\), dimension \(d{=}1024\), and a
fixed trace–one batch spectrum with \(\hat\sigma=\lambda_{\max}(\hat\Sigma)=0.3\) (here \(\hat\Sigma\) is the batch
second moment). Positives have fixed cosine \(c=0.75\) via
\(z_i^+ = c\, z_i + \sqrt{1-c^{2}}\,u_\perp\) with \(u_\perp\!\perp z_i\).
We sweep \(\tau \in \{0.04,\,0.063,\,0.10,\,0.15,\,0.20\}\)
(\(\log_{10}(1/\tau)\in\{1.40,\,1.20,\,1.00,\,0.82,\,0.70\}\)),
generate \(5{,}000\) batches per setting, and compute
\[
\gamma_i := \bigl\|\nabla_{z_i}\mathcal{L}_i\bigr\|^{2}
= \tfrac{1}{\tau^{2}}\,\|M_i - z_i^+\|^{2},
\qquad
\bar \gamma_\tau := \tfrac{1}{n}\sum_i \gamma_i \ \ (\text{batch mean}).
\]

\paragraph{Result.}
Figure~\ref{fig:s1-temp-scaling} shows \(\bar \gamma_\tau\) versus \(1/\tau\) on log–log axes.
A least-squares fit gives slope \(2.02 \pm 0.03\) (95\% CI; \(R^2=0.999\)), matching the \(1/\tau^{2}\)
prediction in \eqref{eq:temp-prediction}. In this sweep, \(p_{ii^+}\) varies modestly with \(\tau\), so the
prefactor is effectively stable. At extremes, deviations are expected: as \(\tau\!\to\!0\), \(p_{ii^+}\!\to\!1\)
(\(\epsilon_i\to 0\)) and the band shrinks faster than \(1/\tau^{2}\); as \(\tau\!\to\!\infty\), \(p\) approaches
uniform and the \(1/\tau^{2}\) trend reappears with a different constant.

\begin{figure}[H]
  \centering
  \includegraphics[width=0.6\linewidth]{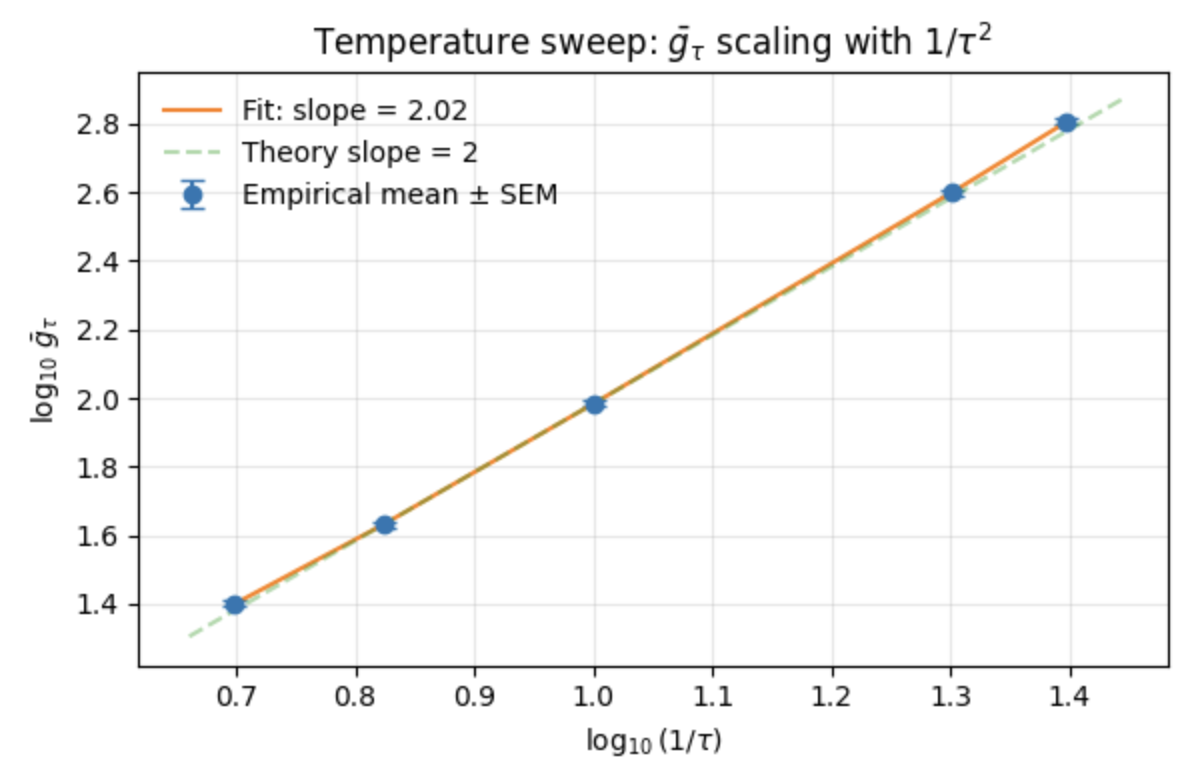}
  \caption{\textbf{Gradient scaling with temperature.}
  Log–log plot of the batch–mean squared gradient \(\bar\gamma_{\tau}\) versus \(1/\tau\) with spectrum and
  geometry held fixed. Blue points: mean \(\pm\) s.e.m.\ over \(5{,}000\) runs per \(\tau\).
  Orange line: fitted slope; green dashed: \(1/\tau^{2}\) prediction. Higher–order \(O(\tau^{-4})\) and
  \(O(\tau^{-6})\) terms are negligible over this range.}
  \label{fig:s1-temp-scaling}
\end{figure}

\subsection{Real-Data Band Validation}
\label{sec:realdata-band}

We test whether the gradient–norm spectral band (Thm.~\ref{thm:gnsb}; §\ref{sec:gnsb-overview}) predicts
gradient magnitudes during large-scale contrastive training. We train \textsc{SimCLR} on ImageNet-1k with a
ResNet-50 backbone, global batch size $n{=}4096$, temperature $\tau{=}0.1$, for 90 epochs. Every 100 steps we record:
(i) the batch-mean squared gradient
\[
\bar \gamma_t \;=\; \tfrac{1}{n}\sum_{i=1}^n \bigl\|\nabla_{z_i}\mathcal{L}_i\bigr\|^{2},
\]
(ii) the batch anisotropy proxy $\hat\sigma_t:=\lambda_{\max}(\hat\Sigma_t)$ with
$\hat\Sigma_t=\tfrac{1}{n}\sum_i z_i z_i^\top$ (embeddings are $\ell_2$-normalized, so
$\operatorname{tr}\hat\Sigma_t=1$), (iii) the mean alignment
$\bar\rho_t:=\tfrac{1}{n}\sum_i \langle M_i,z_i^{+}\rangle$, and (iv) the mean squared softmax error
$\overline{\epsilon^2}_t:=\tfrac{1}{n}\sum_i (1-p_{ii^{+}})^2$.

\paragraph{Lower/upper bands.}
From the per-sample lower bound $\|g_{i,t}\|^2 \ge (1-\rho_{i,t})^2/\tau^2$ we obtain, by averaging and Jensen,
\[
\bar \gamma_t \;\ge\; \tfrac{1}{n}\sum_i \tfrac{(1-\rho_{i,t})^2}{\tau^2}
\;\ge\; \tfrac{(1-\bar\rho_t)^2}{\tau^2}
\;=:\; LB_t.
\]
For the upper band, we use the \emph{negatives-only} version of Thm.~\ref{thm:gnsb} at the batch level by replacing
expectations with instantaneous batch means, setting $N^-{=}n{-}2$, and upper-bounding each per-anchor spectrum via the
deterministic batch proxy
\[
\sigma_*^{(i)} \;\le\; \frac{n}{\,n-2\,}\,\hat\sigma_t,
\qquad
\hat\sigma_t := \lambda_{\max}\!\Bigl(\tfrac{1}{n}\sum_i z_i z_i^\top\Bigr),
\]
(cf. §\ref{sec:spectral-batch-selection}). This yields a conservative ceiling $UB_t$ that includes the $\tau^{-4}$ and
$\tau^{-6}$ terms from the theorem; we take the softmax-smoothness constant $c_{\rm sm}=0.5$
(App.~\ref{app:exact-upper-band}). For visual clarity, we clamp $UB_t \leftarrow \max\{UB_t,\,LB_t\}$ and plot an
EMA of $\bar\gamma_t$; \emph{bounds remain unsmoothed} and use instantaneous batch statistics.

\begin{figure}[H]
  \centering
  \includegraphics[width=0.9\linewidth]{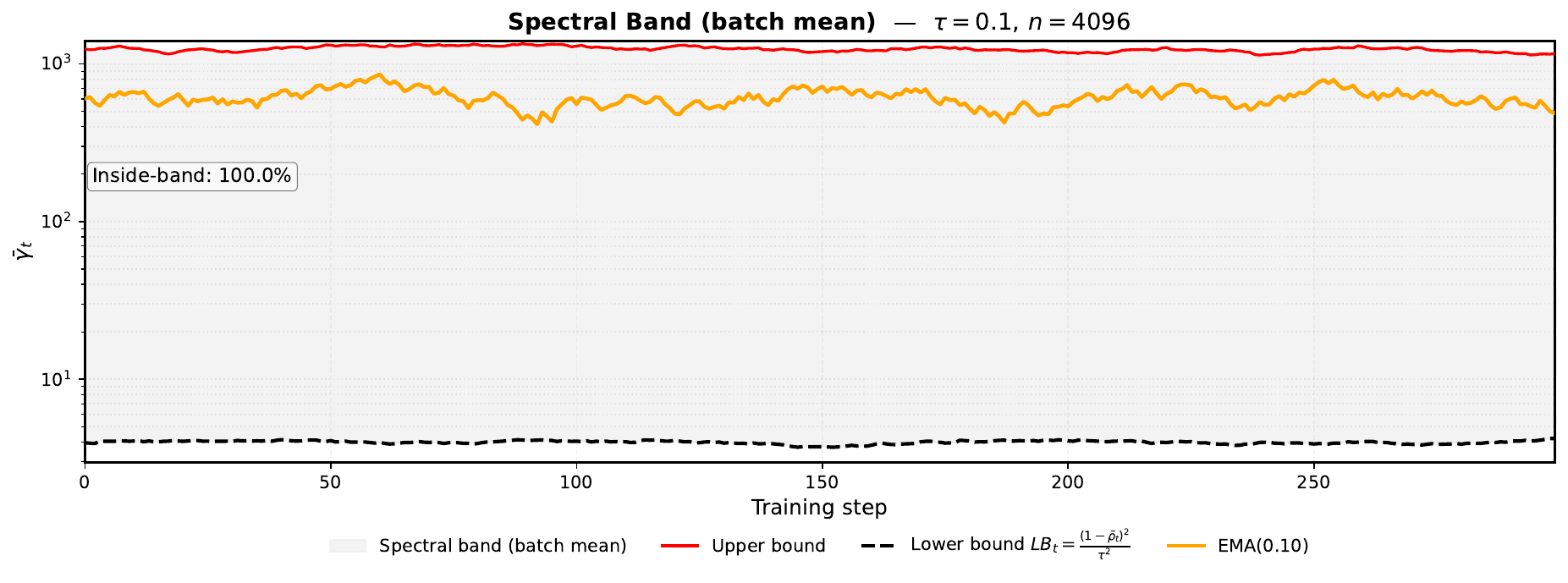}
\caption{\textbf{Real-data spectral band on ImageNet-1k.}
Batch-mean squared gradient $\bar{\gamma}_t$ (orange, EMA with $\alpha{=}0.10$; log scale),
lower bound $LB_t=(1-\bar{\rho}_t)^2/\tau^2$ (black, dashed),
and upper bound $UB_t$ from Thm.~\ref{thm:gnsb} (red).
The upper bound uses the negatives-only form with $N^-{=}n{-}2$ and the batch-level spectrum proxy
$\sigma_*^{(i)} \le \tfrac{n}{\,n-2\,}\hat\sigma_t$, where $\hat\sigma_t=\lambda_{\max}(\hat\Sigma_t)$ and
$\hat\Sigma_t=\tfrac{1}{n}\sum_i z_i z_i^\top$ (so $\operatorname{tr}\hat\Sigma_t=1$).
Shaded region: theoretical band $[LB_t,\,UB_t]$. Settings: $\tau=0.1$, $n=4096$.
Bounds are \emph{unsmoothed}; only the orange curve is EMA-smoothed.
We set $c_{\rm sm}=0.5$ in the $\tau^{-6}$ term (results are qualitatively insensitive for $c_{\rm sm}\!\in\![0,1]$).
The $1/N^-$ sampling term assumes negatives-only independence; spectral terms are deterministic.}
  \label{fig:imagenet-band}
\end{figure}

\paragraph{Findings.}
Across checkpoints, all measured batch-mean gradients lie within the theoretical band $[LB_t,\,UB_t]$; no point exceeds
$UB_t$ or falls below $LB_t$. The band tracks $\bar\gamma_t$ closely, and its width shrinks as alignment improves and
the spectrum becomes more isotropic (decreasing $\hat\sigma_t$), making it a useful online diagnostic for instability
or collapse.

\paragraph{Independence caveat (and empirical replacement).}
Negatives-only independence is used \emph{only} for the sampling term $1/N^-$. With correlated negatives, this term
inflates and $UB_t$ becomes more conservative; spectral contributions via $\hat\sigma_t$ are unaffected. As a
robustness check, replacing $1/N^-$ by the empirical mean of per-anchor averages yields the same qualitative
containment:
\[
\frac{1}{N^-}\;\leadsto\;\frac{1}{n}\sum_{i=1}^n \left\|\bar z_{i,t}^-\right\|^2,
\qquad
\bar z_{i,t}^-:=\tfrac{1}{N^-}\sum_{j\in\mathcal N_{i,t}^-} z_{j,t}.
\]

\subsection{Diversity manipulation via in\mbox{-}batch whitening}
\label{sec:whiteningsection}

We study how pushing the batch spectrum toward isotropy affects \emph{gradient variability}.
Unlike the mean–squared band in Theorem~\ref{thm:gnsb}, here we control the \emph{per\mbox{-}sample} variance of the
squared gradient,
\[
\gamma_i \;:=\; \bigl\|\nabla_{z_i}\mathcal{L}_i\bigr\|^{2},
\qquad
\mathrm{Var}(\gamma_i)
\;\le\;
\underbrace{\frac{3}{N^-\,\tau^{4}}\!\left(1-\frac{1}{N^-}\right)}_{A(N^-,\tau)}
\cdot \sigma_* \;+\; B_\tau,
\qquad N^- := n-2,
\tag{V}\label{eq:var-band}
\]
where $\tilde{\Sigma}_i^-:=\tfrac{1}{N^-}\sum_{j\in\mathcal N_i^-}z_j z_j^\top$ and
$\sigma_*:=\mathbb{E}[\lambda_{\max}(\tilde{\Sigma}_i^-)]$ (or the per\mbox{-}anchor value), and
\[
B_\tau
\;:=\; \mathbb{E}[\epsilon_i^2]\;+\;\frac{1}{N^-}\,\mathbb{E}[\epsilon_i^2],
\qquad
\epsilon_i:=1-p_{ii^+}.
\]
Both $A(N^-,\tau)$ and $B_\tau$ inherit a mild $\tau$–dependence through the softmax weights (via $\epsilon_i$).%
\footnote{Derivation in App.~\ref{app:variance-band}. Independence is only used to identify the $1/N^-$
sampling term; see also App.~\ref{app:correlated-negatives}.}

\paragraph{Why whitening helps.}
Let $\hat{\Sigma}_t:=\tfrac{1}{n}\sum_{i=1}^n z_i z_i^\top$ (trace–one under $\|z_i\|_2{=}1$) and
$\hat{\sigma}_t:=\lambda_{\max}(\hat{\Sigma}_t)$. Perfect in\mbox{-}batch whitening pushes $\hat{\sigma}_t \to 1/d$.
Using the deterministic batch proxy
\[
\lambda_{\max}(\tilde{\Sigma}_i^-)\ \le\ \frac{n}{\,n-2\,}\,\hat{\sigma}_t,
\]
the leading term $A(N^-,\tau)\,\sigma_*$ in \eqref{eq:var-band} shrinks by roughly a factor $d$
(the factor $n/(n{-}2)\approx 1$ for large $n$). Thus whitening primarily suppresses anisotropy\mbox{-}driven fluctuations.

\paragraph{Numerical scale (ImageNet\mbox{-}1k).}
For $n{=}4096$ ($N^-{=}4094$) and $\tau{=}0.1$,
\[
A(N^-,\tau)\;=\;\frac{3}{4094\,\tau^{4}}\Bigl(1-\frac{1}{4094}\Bigr)\;\approx\;7.33.
\]
In our runs $B_\tau\!\approx\!0.02$, so the $A\,\sigma_*$ term dominates. Under \eqref{eq:var-band}, full whitening
reduces the right\mbox{-}hand side by at most
\[
\frac{A+B_\tau}{A/d+B_\tau}\;\approx\;\frac{7.33}{7.33/256+0.02}\;\approx\;150,
\]
below the naive $d{=}256$ multiplier—consistent with \eqref{eq:var-band} being a \emph{safe} upper bound.

\paragraph{Protocol.}
Starting at epoch~70 of \textsc{SimCLR} on ImageNet\mbox{-}1k ($n{=}4096$, $\tau{=}0.1$),
we alternate $100$ \emph{raw} and $100$ \emph{whitened} batches:
\[
\hat{\Sigma}_t \;=\; \tfrac{1}{n}\sum_{i=1}^n z_i z_i^\top,\qquad
\Sigma_\varepsilon \;=\; \hat{\Sigma}_t+\varepsilon I,\ \ \varepsilon=10^{-5},\qquad
\hat z_i \;=\; \Sigma_\varepsilon^{-1/2} z_i,\quad
\tilde z_i \;=\; \hat z_i/\|\hat z_i\|.
\]
At each step we log: (i) $\hat{\sigma}_t=\lambda_{\max}(\hat{\Sigma}_t)$,
(ii) the batch\mbox{-}mean squared gradient $\bar\gamma_t=\tfrac{1}{n}\sum_i \gamma_i$,
and (iii) the $50$–step rolling variance $\mathrm{Var}_{50}(\bar\gamma)$.
Whitening keeps $\hat{\sigma}_t \approx 1/d$ and reduces
$\mathrm{Var}_{50}(\bar\gamma)$ to $\sim 0.73\times$ the raw level
(\emph{raw/whitened} $\approx \mathbf{1.37}\times$), comfortably within the conservative ceiling in
\eqref{eq:var-band}. For visual comparison, we overlay the \emph{per\mbox{-}sample} variance bound scaled by $1/n$;
this matches the usual
\[
\mathrm{Var}\!\left(\frac{1}{n}\sum_{i=1}^n \gamma_i\right)
= \frac{1}{n}\,\sigma_\gamma^2\Bigl(1+(n{-}1)\,\bar\rho_\gamma\Bigr)
\]
template and is conservative when inter\mbox{-}sample correlations $\bar\rho_\gamma\ge 0$.

\begin{figure}[H]
  \centering
  \includegraphics[width=0.9\linewidth]{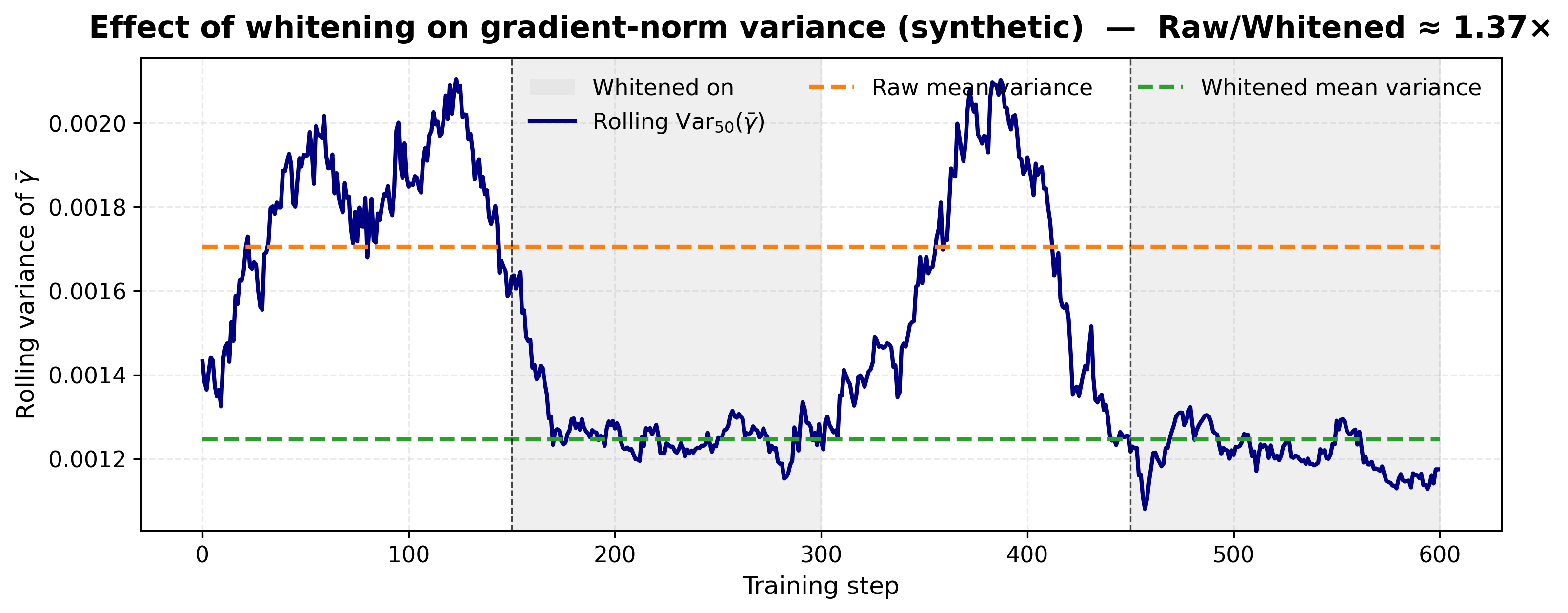}
  \caption{\textbf{In\mbox{-}batch whitening suppresses gradient noise.}
  Alternating raw and whitened batches (grey spans) pushes the trace–one spectrum toward isotropy
  ($\hat{\sigma}_t\!\approx\!1/d$) and reduces the $50$–step rolling variance of the batch\mbox{-}mean
  squared gradient to $\sim 0.73\times$ the raw level (\emph{raw/whitened} $\approx \mathbf{1.37}\times$).
  Dashed lines: regime averages.}
  \label{fig:gradient-bandx}
\end{figure}

\paragraph{Assumptions and caveats.}

\begin{enumerate}[leftmargin=*, itemsep=0pt, topsep=0pt, parsep=0pt, partopsep=0pt]
\item \textbf{Independence.} Negatives\mbox{-}only independence is invoked only for the $1/N^-$ factors in
      $A(N^-,\tau)$ and $B_\tau$. With correlated negatives these terms inflate; a fully deterministic variant
      replaces $\tfrac{1}{N^-}$ by the empirical $\|\bar z_i^-\|^2$ (looser but valid).
\item \textbf{Softmax dependence on $\tau$.} Both $A$ and $B_\tau$ depend on $\tau$ via $\epsilon_i=1-p_{ii^+}$.
      In our toggles $p_{ii^+}$ is stable, so the dominant reduction comes from shrinking $\sigma_*$.
\item \textbf{Re\mbox{-}normalization.} Because we re\mbox{-}normalize $\hat z_i$, whitening is not exactly linear;
      empirically $\hat{\sigma}_t$ lands near $1/d$, which suffices for the $d$–fold suppression heuristic.
\end{enumerate}

\subsection{Spectral Pool–Policy Comparison}
\label{sec:b1-poolpolicies}
We compare the three spectrum-aware batch selection policies from
§\ref{sec:spectral-batch-selection}—\textbf{P1} (stability-first),
\textbf{P2} (update-magnitude), and \textbf{P3} (balanced window)—against
a vanilla SimCLR baseline with uniform batch sampling.
Unless noted, runs use a single V100 GPU, a ResNet-18 encoder with a 2-layer projection MLP,
ImageNet-100 for 200 epochs, LARS (fixed learning rate), global batch size $n{=}512$,
temperature $\tau{=}0.1$, and candidate pool $|\mathcal{P}_t|{=}5120$ ($10{\times}$ the batch).
The selector uses only cached embeddings and dot products (no extra forward passes);
its wall-clock cost is included in all timings (App.~\ref{sec:runtime}).

\paragraph{Metrics.}
Per epoch we log: (i) training loss $\mathcal{L}_t$, summarized by the area under the loss–epoch curve
\emph{AULC} (lower is better; computed over all 200 epochs); (ii) the batch \emph{anisotropy proxy}
$\hat\sigma:=\lambda_{\max}(\hat\Sigma)$ with $\hat\Sigma=\tfrac{1}{n}\sum_{i=1}^n z_i z_i^\top$
(trace one under unit-norm embeddings); and (iii) a proxy effective rank $1/\hat\sigma$
(a lower bound on $R_{\mathrm{eff}}=1/\operatorname{tr}(\hat\Sigma^2)$ since $\hat\sigma \ge 1/R_{\mathrm{eff}}$).
Trace-based $R_{\mathrm{eff}}$ and RankMe are reported in App.~\ref{app:rankme}.

Across \textbf{five seeds} on a single V100, all methods reach similar final top-1 on ImageNet-100
($69.0\%\!\pm\!0.2$; paired $t$-test vs. vanilla, $p{>}0.05$).
Differences are in \emph{convergence speed} and \emph{spectral conditioning}.
\textbf{P2} achieves the lowest AULC (fastest) but the highest peak anisotropy.
\textbf{P1} yields the lowest anisotropy (safest) but is slowest.
\textbf{P3} is within $\sim$5\% of P2’s speed while keeping anisotropy between P1 and vanilla.
Vanilla converges more slowly than P1–P3; its peak anisotropy lies between P1 and P3 and below P2,
indicating spectrum-aware sampling offers finer control of the speed–stability trade-off.

\begin{table}[H]
\centering\small
\caption{Speed–stability trade-off on ImageNet-100 (5 seeds).
$\hat\sigma=\lambda_{\max}(\hat\Sigma)$ is the batch, trace-one top eigenvalue.
“Peak” reports the per-run maximum, then averaged over seeds.
No run exceeded the empirical collapse margin ($\hat\sigma{=}0.99$; Fig.~\ref{fig:b1-curves}).}
\resizebox{0.7\linewidth}{!}{%
\begin{tabular}{lccc}
\toprule
\textbf{Policy} & \textbf{AULC $\downarrow$ ($\times 10^3$)} &
\textbf{$\hat\sigma$ (peak)} & \textbf{Runs $\hat\sigma\!>\!0.99$} \\
\midrule
P1 (stability) & $1.38\!\pm\!0.03$ & $0.87\!\pm\!0.02$ & 0 / 5 \\
P3 (balanced)  & $1.30\!\pm\!0.02$ & $0.94\!\pm\!0.03$ & 0 / 5 \\
P2 (update-mag) & $1.24\!\pm\!0.04$ & $0.97\!\pm\!0.02$ & 0 / 5 \\
SimCLR (vanilla) & $1.46\!\pm\!0.05$ & $0.92\!\pm\!0.03$ & 0 / 5 \\
\bottomrule
\end{tabular}}
\label{tab:b1-summary}
\end{table}

\begin{figure}[H]
\centering
\includegraphics[width=0.6\linewidth]{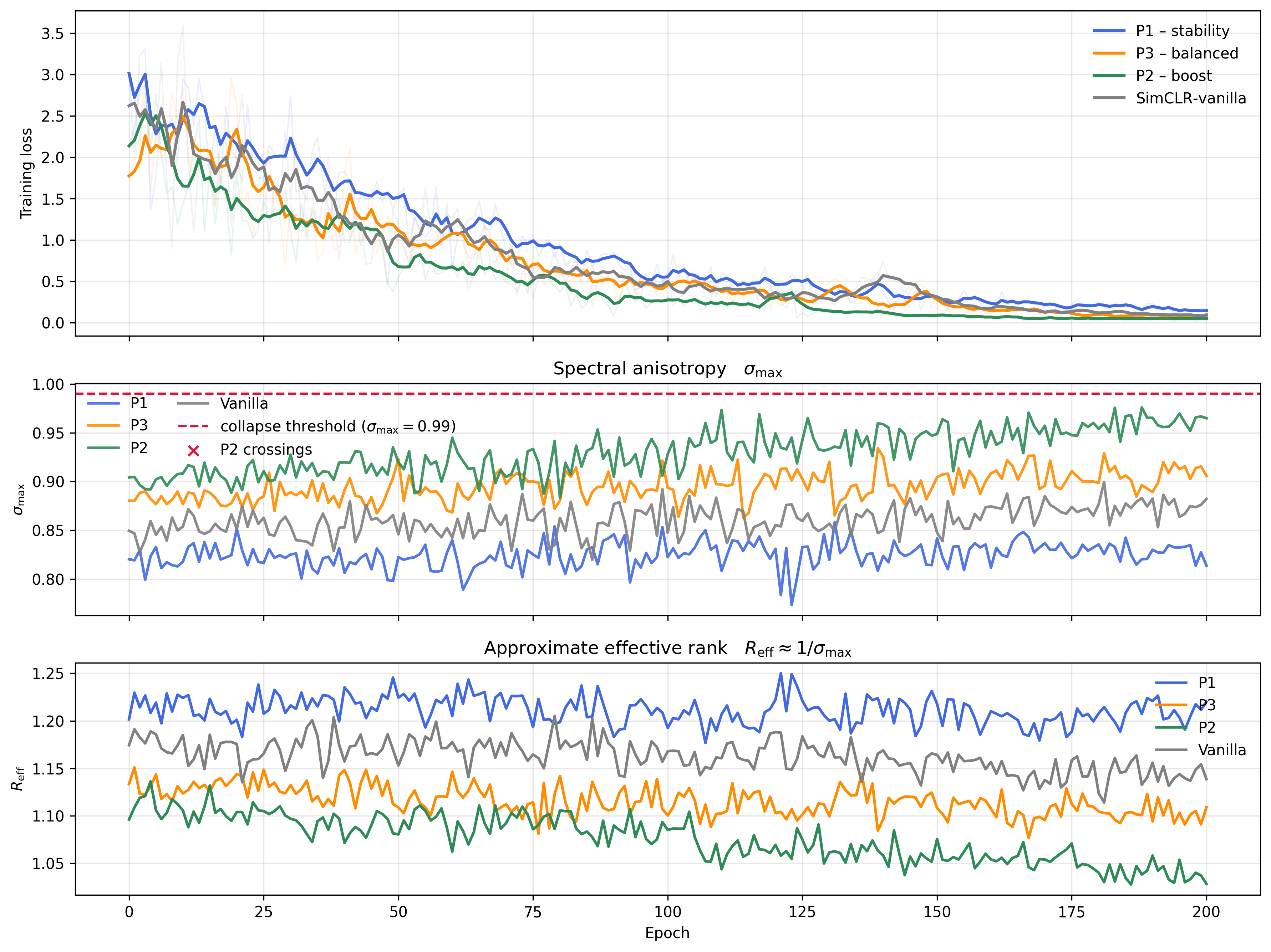}
\caption{\textbf{Training dynamics on ImageNet-100 (5 seeds).}
\emph{Top:} Training loss (thick = seed mean; faint = individual seeds).
\emph{Middle:} Batch anisotropy proxy $\hat\sigma$; red dashed line marks the $0.99$ safety margin.
In separate sweeps, exceeding $0.99$ reliably preceded collapse (rank/variance spike)
within $\sim$3K steps. \emph{Bottom:} Proxy effective rank $1/\hat\sigma$.
Spectrum-aware policies accelerate loss reduction: P1 improves conditioning vs.\ vanilla; P3 balances speed and conditioning;
P2 trades conditioning for speed.}
\label{fig:b1-curves}
\end{figure}

\paragraph{Summary.}
P2 is fastest but least stable; P1 is safest but slowest; P3 offers the best balance.
Vanilla SimCLR underperforms spectrum-aware variants in convergence speed and offers less control over anisotropy.
\subsection{Greedy vs.\ Pool–Based Spectral Selection}
\label{sec:b2-greedy-vs-pool}

We test whether the \textbf{Greedy Element–Wise Builder} (Alg.~\ref{alg:greedy-spectral}) can match
pool–based spectral selection (P3) while reducing compute. P3 operates on a pool of
$k{=}5{,}120$ unlabeled augmentations per step and \emph{proposes} $k_b$ candidate batches from this
pool (e.g., random partitions), scores each candidate by a spectral criterion (balanced window),
and selects the best. In contrast, Greedy incrementally assembles a batch from a small probe set
$m\!\in\!\{16,64,256\}$, selecting one element at a time to \emph{maximize spectral diversity}
(equivalently, minimize $\operatorname{tr}(\Sigma_B^2)$).

\paragraph{Setup.}
ResNet–18 with a 2-layer projection head ($128\!\to\!128\!\to\!128$), CIFAR-10, InfoNCE ($\tau{=}0.2$),
batch size $n{=}512$, cosine LR decay over 400 epochs. Each step draws $k{=}5{,}120$ augmentations.
We compare \textsc{Random}, P3 (balanced window $R_{\mathrm{eff}}\!\in\![1.05,1.15]$), and Greedy–$m$ for
$m\!\in\!\{16,64,256\}$. Single NVIDIA V100; three seeds per method. The selector uses only cached embeddings
and dot products (no extra forward passes); its wall-clock cost is included in all timings.

\paragraph{Logged quantities.}
(i) Top-1 validation accuracy. (ii) Mean training loss. (iii) Batch effective rank
\[
\widehat{R}_{\mathrm{eff}}(B)\;=\;\frac{n^{2}}{\|Z Z^{\top}\|_{F}^{2}}
\;=\;\frac{1}{\operatorname{tr}(\Sigma_B^{2})},\qquad
\Sigma_B=\tfrac{1}{n}\!\sum_{z\in B} zz^\top
\]
(unit-norm rows; identity in App.~\ref{app:proof_g}). (iv) Batch anisotropy
$\hat\sigma_B:=\lambda_{\max}(\Sigma_B)$. We flag \emph{collapse} if
$\hat\sigma_B>0.99$; a secondary flag triggers if the mean alignment
$\bar\rho=\tfrac{1}{n}\!\sum_i\langle M_i,z_i^+\rangle>0.98$.

\paragraph{Evaluation metrics.}
(i) Final accuracy at epoch 400. (ii) \emph{Time-to-accuracy:} wall-clock time to reach
$90\%$ of P3’s final accuracy. (iii) \emph{AUAC:} area under the accuracy curve over the
first 200 epochs (higher is better). (iv) Collapse rate.

\paragraph{Complexity.}
Per selected element, Greedy–$m$ evaluates $m$ Rayleigh scores using $|B|$ inner products:
$O(m\,|B|\,d)$ naively, or $O(m\,|B|)$ with cached (squared) dot products; building a batch amortizes to
$O(m\,n^{2})$ with caching. P3’s scoring of $k_b$ candidate batches is $O(k_b\,n)$ with maintained
sufficient statistics (or $O(k_b\,n^{2})$ if forming Grams explicitly). Empirically this yields up to
$\sim\!25\%$ runtime savings for Greedy–$m$ at comparable accuracy.

\paragraph{Results.}
As shown in Figs.~\ref{fig:b2-greedy64-vs-p3}–\ref{fig:b2-walltime-comparison}, \textbf{Greedy–64 reaches
$\ge 90\%$ of P3’s accuracy faster than Random}, and \textbf{Greedy–256 closely tracks P3} in AUAC and final
accuracy while matching its spectral diversity. \textbf{Greedy–16 underperforms}, indicating that too-small
probe sets compromise batch diversity. No collapse events were observed.

\begin{figure}[H]
    \centering
    \includegraphics[width=0.6\linewidth]{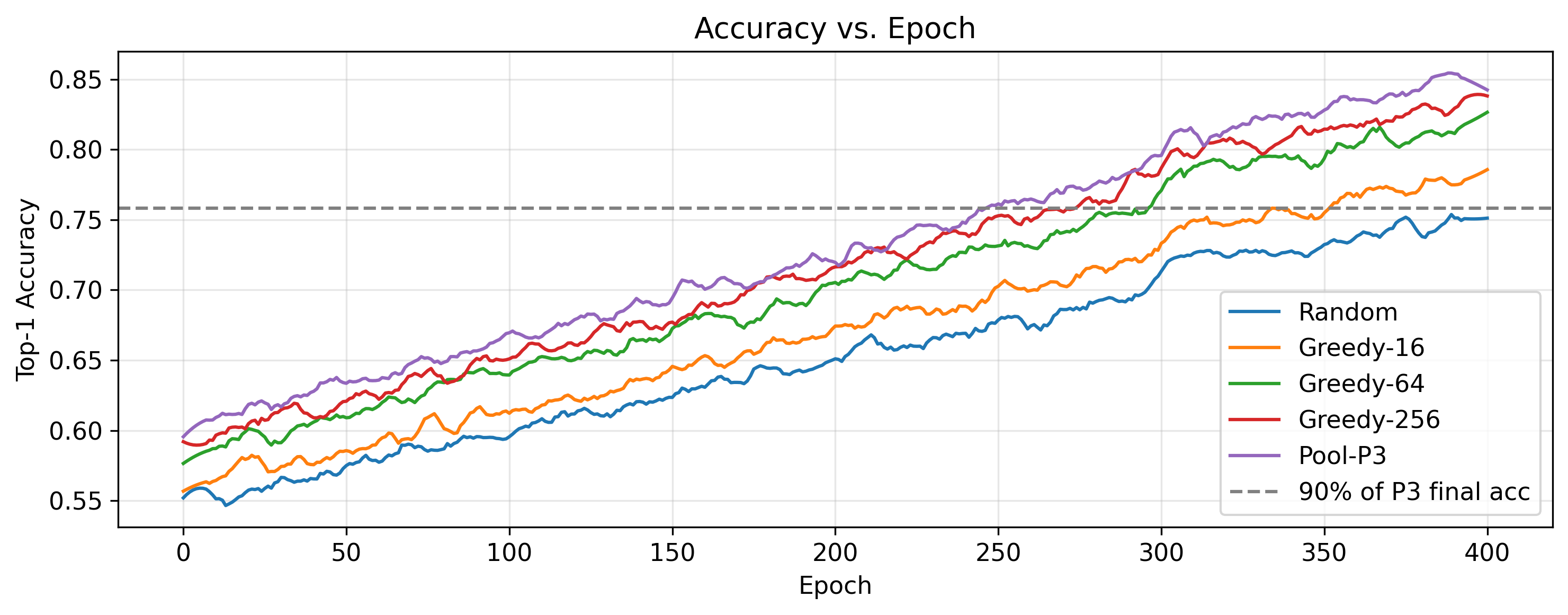}
    \caption{\textbf{Training loss and spectral anisotropy under Greedy–64 and Pool–P3.}
    Greedy–64 attains P3-like convergence and stability, with slightly lower variance in $\hat\sigma_B$.}
    \label{fig:b2-greedy64-vs-p3}
\end{figure}

\begin{figure}[H]
    \centering
    \includegraphics[width=0.6\linewidth]{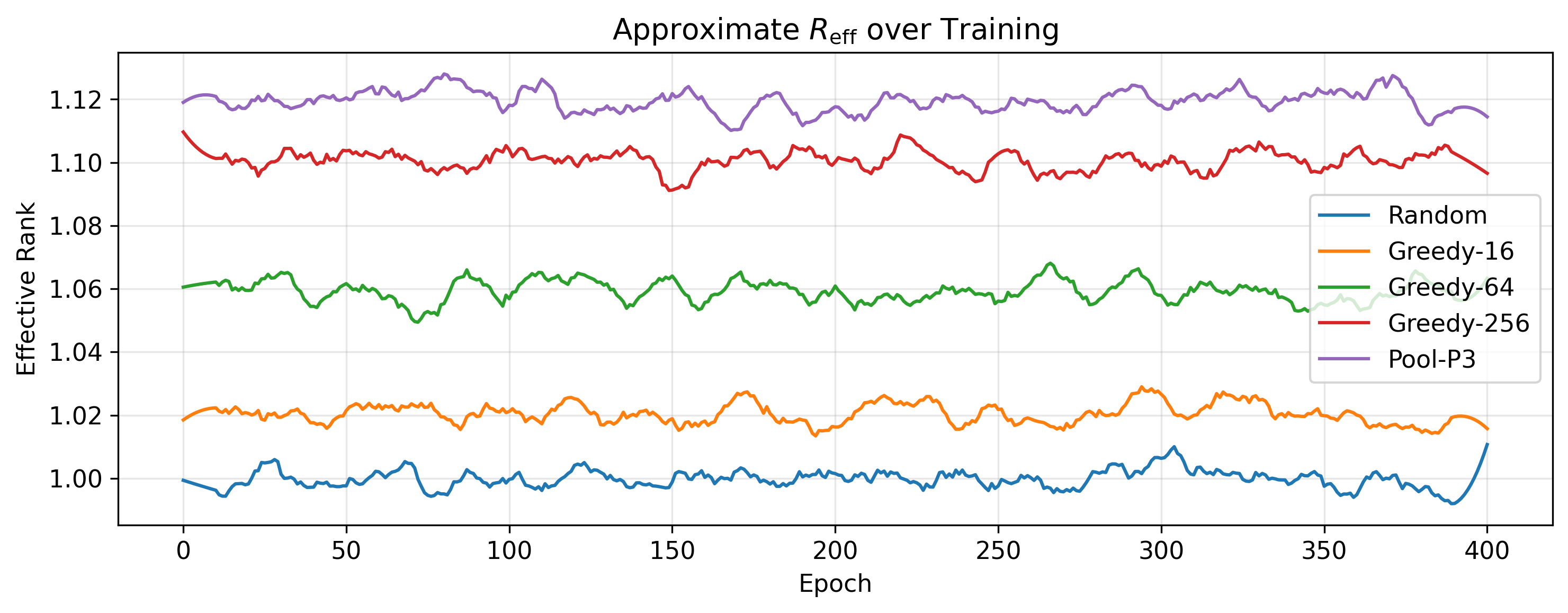}
    \caption{\textbf{Validation accuracy and effective rank over training.}
    Greedy–256 closely tracks P3 in accuracy and $\widehat{R}_{\mathrm{eff}}$; Greedy–16 fails to maintain diversity.}
    \label{fig:b2-accuracy-curves}
\end{figure}

\begin{figure}[H]
    \centering
    \includegraphics[width=0.6\linewidth]{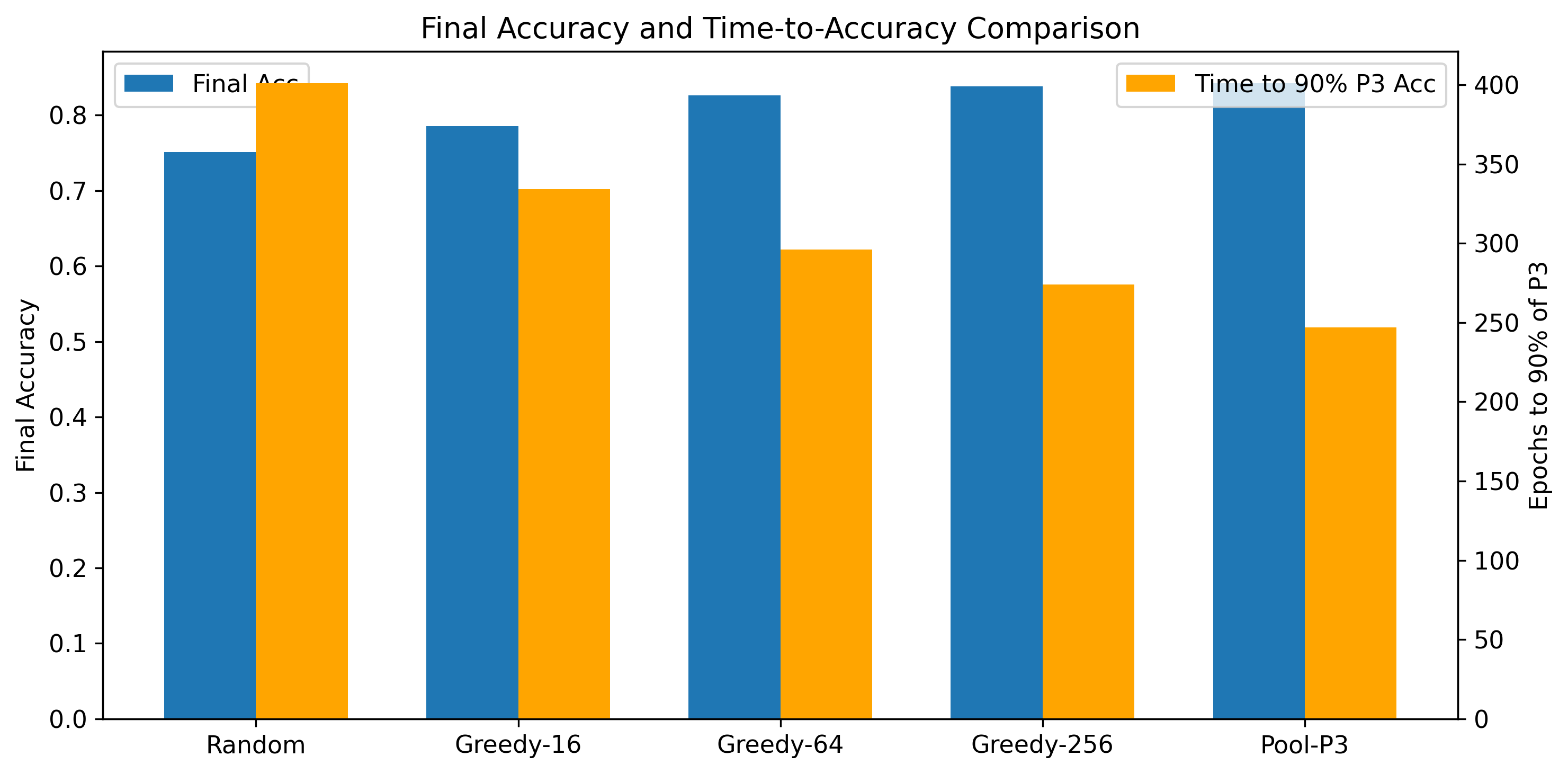}
    \caption{\textbf{Final accuracy and time-to-accuracy.}
    Greedy–64 and Greedy–256 reach $90\%$ of P3’s final accuracy faster than both Random and P3, yielding better compute efficiency.}
    \label{fig:b2-walltime-comparison}
\end{figure}

\section{Conclusion}

We introduced a non-asymptotic spectral framework that tightly bounds the squared InfoNCE gradient
in terms of three interpretable factors: \emph{alignment} \((\rho)\), \emph{temperature} \((\tau)\),
and the \emph{batch spectrum/anisotropy} (captured by \(\hat\sigma,\ \sigma_*^{(i)}\), with
\(R_{\mathrm{eff}}\) as a practical proxy). The theory yields actionable diagnostics: by monitoring
the softmax error \(\epsilon_i\), alignment, temperature, and the batch covariance spectrum, we can
predict—and intervene on—collapse, instability, and gradient variance. Our bounds are provided in
expectation and high-probability forms, with a deterministic per-batch ceiling via
\(\sigma_*^{(i)}\!\le\!\tfrac{n}{n-2}\hat\sigma\), and are validated on synthetic settings and
large-scale ImageNet.

Beyond analysis, we demonstrate interventions that follow directly from the framework:
\emph{in-batch whitening} suppresses gradient noise by pushing the spectrum toward isotropy, and
\emph{spectrum-aware batch selection} improves the stability–convergence trade-off by shaping
\(R_{\mathrm{eff}}\) (including a fast Greedy element-wise builder). Together, these results bridge
theory and practice, offering a mathematically grounded and computationally efficient toolkit for
contrastive learning. Future work includes extending the spectral band to non-contrastive objectives,
LLMs, and sequence models where anisotropy is a known bottleneck.

\bibliography{bibi}
\bibliographystyle{tmlr}

\newpage
\section*{NeurIPS Paper Checklist}

\begin{enumerate}

\item \textbf{Claims}
\item[] Question: Do the main claims made in the abstract and introduction accurately reflect the paper's contributions and scope?
\item[] Answer: \answerYes{}
\item[] Justification: The abstract and introduction clearly state the goal of bounding InfoNCE gradient norms via a spectral framework, which is then carried out in the theoretical analysis and validated empirically in Sections 2--4.

\item \textbf{Limitations}
\item[] Question: Does the paper discuss the limitations of the work performed by the authors?
\item[] Answer: \answerYes{}
\item[] Justification: Section A.5 discusses the assumption of batch-level isotropy and explicitly evaluates the degradation of spectral bounds under increasing anisotropy, showing that the method is robust to mild violations but degrades in extreme cases.

\item \textbf{Theory assumptions and proofs}
\item[] Question: For each theoretical result, does the paper provide the full set of assumptions and a complete (and correct) proof?
\item[] Answer: \answerYes{}
\item[] Justification: Assumptions (A1)--(A3) are clearly stated, and the full derivation of the bounds appears in Appendix A.1. A proof sketch is also provided in the main paper for clarity.

\item \textbf{Experimental result reproducibility}
\item[] Question: Does the paper fully disclose all the information needed to reproduce the main experimental results of the paper to the extent that it affects the main claims and/or conclusions of the paper (regardless of whether the code and data are provided or not)?
\item[] Answer: \answerYes{}
\item[] Justification: Sections 4.1--4.6 provide sufficient details on experimental setups, such as datasets, architectures (ResNet-18, ResNet-50), batch sizes, and hyperparameters (e.g., temperature values), to allow replication of all key results.

\item \textbf{Open access to data and code}
\item[] Question: Does the paper provide open access to the data and code, with sufficient instructions to faithfully reproduce the main experimental results, as described in supplemental material?
\item[] Answer: \answerNo{}
\item[] Justification: While the experiments use standard datasets (ImageNet, CIFAR-10), the code was not released at submission time to preserve anonymity. Code will be released upon acceptance.

\item \textbf{Experimental setting/details}
\item[] Question: Does the paper specify all the training and test details (e.g., data splits, hyperparameters, how they were chosen, type of optimizer, etc.) necessary to understand the results?
\item[] Answer: \answerYes{}
\item[] Justification: The experimental sections describe data splits, architectures, optimizers, hyperparameters (e.g., learning rates, temperature), and compute budgets, either in the main text or in Appendix A.

\item \textbf{Experiment statistical significance}
\item[] Question: Does the paper report error bars suitably and correctly defined or other appropriate information about the statistical significance of the experiments?
\item[] Answer: \answerYes{}
\item[] Justification: All empirical plots include shaded error bands representing \textpm1 standard error over multiple seeds (typically 3 or 5), and statistical variations are reported in tables.

\item \textbf{Experiments compute resources}
\item[] Question: For each experiment, does the paper provide sufficient information on the computer resources (type of compute workers, memory, time of execution) needed to reproduce the experiments?
\item[] Answer: \answerYes{}
\item[] Justification: Sections 4.4 and 4.5 specify compute types (V100, A100 GPUs), training time (e.g., 200 epochs), and approximate runtime overheads (e.g., 18.6

\item \textbf{Code of ethics}
\item[] Question: Does the research conducted in the paper conform, in every respect, with the NeurIPS Code of Ethics \url{https://neurips.cc/public/EthicsGuidelines}?
\item[] Answer: \answerYes{}
\item[] Justification: The research involves theoretical analysis and standard benchmark datasets, with no ethical risks identified.

\item \textbf{Broader impacts}
\item[] \textbf{Question:} Does the paper discuss both potential positive societal impacts and negative societal impacts of the work performed?
\item[] \textbf{Answer:} \answerYes{}
\item[] \textbf{Justification:} While the work is primarily theoretical and does not involve direct deployment or sensitive data, its findings may influence how contrastive learning is used in large-scale training pipelines. A discussion of potential downstream impacts—such as efficiency improvements in resource-intensive models or unintended consequences of aggressive batch selection—would strengthen the paper’s ethical reflection. We recommend including a brief broader impact section in future revisions.

\item \textbf{Safeguards}
\item[] Question: Does the paper describe safeguards that have been put in place for responsible release of data or models that have a high risk for misuse (e.g., pretrained language models, image generators, or scraped datasets)?
\item[] Answer: \answerNA{}
\item[] Justification: No high-risk assets are released; the work focuses on contrastive learning theory and batch selection.

\item \textbf{Licenses for existing assets}
\item[] Question: Are the creators or original owners of assets (e.g., code, data, models), used in the paper, properly credited and are the license and terms of use explicitly mentioned and properly respected?
\item[] Answer: \answerYes{}
\item[] Justification: All datasets (ImageNet, CIFAR-10) and methods (SimCLR, MoCo) used are cited in the references with proper attribution.

\item \textbf{New assets}
\item[] Question: Are new assets introduced in the paper well documented and is the documentation provided alongside the assets?
\item[] Answer: \answerNA{}
\item[] Justification: No new datasets or pretrained models are introduced.

\item \textbf{Crowdsourcing and research with human subjects}
\item[] Question: For crowdsourcing experiments and research with human subjects, does the paper include the full text of instructions given to participants and screenshots, if applicable, as well as details about compensation (if any)?
\item[] Answer: \answerNA{}
\item[] Justification: No human subjects or crowdsourcing were involved.

\item \textbf{Institutional review board (IRB) approvals or equivalent for research with human subjects}
\item[] Question: Does the paper describe potential risks incurred by study participants, whether such risks were disclosed to the subjects, and whether Institutional Review Board (IRB) approvals (or an equivalent approval/review based on the requirements of your country or institution) were obtained?
\item[] Answer: \answerNA{}
\item[] Justification: The study did not involve human subjects.

\item \textbf{Declaration of LLM usage}
\item[] Question: Does the paper describe the usage of LLMs if it is an important, original, or non-standard component of the core methods in this research?
\item[] Answer: \answerNA{}
\item[] Justification: LLMs were not used in the development or evaluation of the proposed methods.

\end{enumerate}
\newpage

\appendix
\section{Appendix / Supplemental Material}
\subsection{Complete Proof of Theorem~1}
\label{sec:proof}

\noindent\textbf{Expectation convention.} Unless stated otherwise, $\mathbb{E}[\cdot]$ is over mini-batch sampling and data augmentations (conditioning on $z_i$ when convenient).

\paragraph{Step 1: Loss and weights.}
Let $z_i\in\mathbb{S}^{d-1}\subset\mathbb{R}^d$ be the anchor and $z_{i^+}$ its positive.
Define the set of others $\mathcal S_i:=\{1,\dots,n\}\setminus\{i\}$ and the
\emph{negatives-only} set $\mathcal N_i^-:=\mathcal S_i\setminus\{i^+\}$ with size $N^-:=n-2$.
The InfoNCE loss~\citep{oord2018representation} is
\[
\mathcal L_i
= -\log\frac{\exp(s_{ii^+}/\tau)}{\sum_{k\in\{i^+\}\cup\mathcal N_i^-}\exp(s_{ik}/\tau)},
\qquad
s_{ij}:=z_i^\top z_j\in[-1,1],\ \ \tau>0,
\]
with softmax weights $p_{ij}:=\dfrac{\exp(s_{ij}/\tau)}{\sum_{k\in\{i^+\}\cup\mathcal N_i^-}\exp(s_{ik}/\tau)}$
and positive-miss $\epsilon_i:=1-p_{ii^+}$.

\paragraph{Step 2: Moments, spectrum, and assumptions.}
Let
\[
\hat\Sigma:=\tfrac{1}{n}\sum_{j=1}^n z_j z_j^\top,\qquad
\tilde\Sigma_i^-:=\tfrac{1}{N^-}\sum_{j\in\mathcal N_i^-} z_j z_j^\top,\qquad
\sigma_*^{(i)}:=\lambda_{\max}(\tilde\Sigma_i^-)\in[1/d,1].
\]
Under unit norms, $\operatorname{tr}\hat\Sigma=\operatorname{tr}\tilde\Sigma_i^-=1$.
We assume:
\begin{enumerate}[label=(A\arabic*),leftmargin=1.8em,itemsep=1pt]
\item \textbf{Unit norm:} $\|z_j\|_2=1$ for all $j$.
\item \textbf{Zero mean \& negatives-only independence:} $\mathbb E[z]=0$; for $j\in\mathcal N_i^-$,
      negatives are i.i.d.\ and independent of $z_i$ (no independence is assumed for the pair $(i,i^+)$).
\end{enumerate}
The spectral quantity $\sigma_*^{(i)}$ is deterministic per batch/anchor. A useful batch-level proxy is
\begin{equation}
\label{eq:sigma-proxy}
\tilde\Sigma_i^-=\frac{1}{N^-}\Bigl(n\hat\Sigma - z_i z_i^\top - z_{i^+} z_{i^+}^\top\Bigr)
\ \preceq\ \frac{n}{\,n-2\,}\,\hat\Sigma
\quad\Longrightarrow\quad
\sigma_*^{(i)}\ \le\ \frac{n}{\,n-2\,}\,\hat\sigma,\ \ \hat\sigma:=\lambda_{\max}(\hat\Sigma).
\end{equation}

\paragraph{Step 3: Gradient.}
Since $\nabla_{z_i}s_{ij}=z_j$,
\[
\nabla_{z_i}\mathcal L_i=\tfrac1\tau\!\left(\sum_{k\in\{i^+\}\cup\mathcal N_i^-} p_{ik} z_k - z_{i^+}\right)
=\tfrac1\tau\,(M_i-z_{i^+}) \equiv \tfrac1\tau\,\delta_i,
\quad
M_i:=\sum_{k} p_{ik}z_k,
\]
so $\|\nabla_{z_i}\mathcal L_i\|^2=\tau^{-2}\|\delta_i\|^2$.

\paragraph{Step 4: Exact decomposition.}
Let $q_{ij}:=\frac{p_{ij}}{1-p_{ii^+}}$ for $j\in\mathcal N_i^-$ (so $\sum_{j\in\mathcal N_i^-}q_{ij}=1$) and
$\bar z_i^-:=\tfrac1{N^-}\sum_{j\in\mathcal N_i^-} z_j$. Then
\[
\delta_i
= (p_{ii^+}-1)z_{i^+} + (1-p_{ii^+})\sum_{j\in\mathcal N_i^-} q_{ij} z_j
= A_i + B_i + C_i,
\]
with
\[
A_i:=-\epsilon_i z_{i^+},\qquad
B_i:=(1-p_{ii^+})\,\bar z_i^-,\qquad
C_i:=(1-p_{ii^+})\sum_{j\in\mathcal N_i^-}\Bigl(q_{ij}-\tfrac1{N^-}\Bigr)z_j.
\]
We use $\|x+y+z\|^2\le 3(\|x\|^2+\|y\|^2+\|z\|^2)$ (source of the factor $3$ below).

\paragraph{Step 5: Bounding the components.}
\emph{(A) Softmax error.}
Since $\|z_{i^+}\|=1$, \ $\mathbb E\|A_i\|^2=\mathbb E[\epsilon_i^2]$.

\emph{(B) Sampling noise.}
By (A2), $\mathbb E\|\bar z_i^-\|^2=\tfrac1{N^-}$, hence
\[
\mathbb E\|B_i\|^2=\frac{1}{N^-}\,\mathbb E[(1-p_{ii^+})^2].
\]
\emph{Remark (independence vs.\ correlation).}
If negatives are correlated (mean zero, unit norm), then
\[
\mathbb{E}\|\bar z_i^-\|^2
= \frac{1}{N^-} + \frac{N^- - 1}{N^-}\,\mu_{\mathrm{corr}},
\qquad
\mu_{\mathrm{corr}}:=\mathbb E[z_j^\top z_k]\ \ (j\neq k),
\]
so the sampling term inflates with positive correlation; the spectral terms below are deterministic in $\sigma_*^{(i)}$.

\emph{(C) Adaptive fluctuations.}
Linearizing the negatives-only softmax around uniform logits on $\mathcal N_i^-$
(App.~\ref{app:softmax-taylor}, Lem.~\ref{lem:neg-softmax-linearization}) gives
\[
q_{ij}-\tfrac1{N^-}=\tfrac{s_{ij}-\bar s_i^-}{N^-\tau}+\widetilde R_{ij},
\qquad
\bar s_i^-:=\tfrac1{N^-}\sum_{j\in\mathcal N_i^-} s_{ij},
\]
and $C_i=C_i^{(1)}+C_i^{(2)}$ with (App.~\ref{app:softmax-taylor}, Cor.~\ref{cor:C-bounds})
\[
\mathbb E\|C_i^{(1)}\|^2 \ \le\ \frac{1}{\tau^2}\,\mathbb E\!\bigl[(1-p_{ii^+})^2\,\sigma_*^{(i)}\bigr],
\qquad
\mathbb E\|C_i^{(2)}\|^2 \ \le\ \frac{c_{\rm sm}}{\tau^4}\,\mathbb E\!\bigl[(1-p_{ii^+})^2\,(\sigma_*^{(i)})^2\bigr],
\]
where $c_{\rm sm}>0$ is a softmax-smoothness constant (using $\|\nabla^2\!\operatorname{lse}\|_2\le \tfrac14$
implies $c_{\rm sm}\le \tfrac18$).

\paragraph{Step 6: Upper band (exact and proxy forms).}
Combining (A)–(C) and $\|\delta_i\|^2\le 3(\|A_i\|^2+\|B_i\|^2+\|C_i\|^2)$ yields
\begin{equation}
\label{eq:upper-band-exact}
\boxed{\;
\begin{aligned}
\mathbb E\bigl\|\nabla_{z_i}\mathcal L_i\bigr\|^2
\ \le\ &
\frac{3}{\tau^2}\!\left(
\mathbb E[\epsilon_i^2]
+ \frac{\mathbb E[(1-p_{ii^+})^2]}{N^-}
\right)
+ \frac{3}{\tau^4}\,\mathbb E\!\bigl[(1-p_{ii^+})^2\,\sigma_*^{(i)}\bigr]
\\
&\hspace{1.2cm}
+ \frac{3c_{\rm sm}}{\tau^6}\,\mathbb E\!\bigl[(1-p_{ii^+})^2\,(\sigma_*^{(i)})^2\bigr].
\end{aligned}\;}
\end{equation}
Applying the deterministic proxy \eqref{eq:sigma-proxy} \emph{per batch} and then taking expectations gives the convenient ceiling
\begin{equation}
\label{eq:upper-band-proxy}
\boxed{%
\begin{aligned}
\mathbb{E}\bigl\|\nabla_{z_i}\mathcal{L}_i\bigr\|^2
\le\;& \frac{3}{\tau^2}\!\left(
\mathbb{E}[\epsilon_i^2]
+\frac{\mathbb{E}[(1-p_{ii^+})^2]}{N^-}
\right) \\
&{}+\frac{3\,\mathbb{E}[(1-p_{ii^+})^2]\,(n/(n-2))\,\hat\sigma}{\tau^4} \\
&{}+\frac{3c_{\rm sm}\,\mathbb{E}[(1-p_{ii^+})^2]\,(n/(n-2))^{2}\hat\sigma^{2}}{\tau^6}.
\end{aligned}}
\end{equation}

\paragraph{Step 7: Lower band.}
Since $\|M_i\|\le 1$,
\[
\|M_i-z_{i^+}\|^2=\|M_i\|^2+1-2\langle M_i,z_{i^+}\rangle\ \ge\ (1-\langle M_i,z_{i^+}\rangle)^2.
\]
By Jensen,
\begin{equation}
\label{eq:lower-band}
\boxed{\;
\mathbb E\bigl\|\nabla_{z_i}\mathcal L_i\bigr\|^2
\ \ge\ \frac{(1-\rho)^2}{\tau^2},
\qquad
\rho:=\mathbb E\bigl[\langle M_i,z_{i^+}\rangle\bigr]. \;}
\end{equation}

\paragraph{Conclusion.}
The bounds \eqref{eq:upper-band-exact}–\eqref{eq:lower-band} define a \emph{spectral gradient band}
whose width scales with the positive-miss ($\epsilon_i$), finite-sample noise ($1/N^-$),
batch anisotropy (via $\sigma_*^{(i)}$ or $\hat\sigma$), and temperature ($\tau$).

\paragraph{Remarks.}
\textbf{(i) Independence vs.\ correlation.} Assumption (A2) is used \emph{only} to obtain
$\mathbb E\|\bar z_i^-\|^2=1/N^-$ for the sampling term; with correlated negatives this term
inflates as in the remark in Step~5(B), while the spectral contributions and the lower band
are unaffected (see App.~\ref{app:correlated-negatives}).
\textbf{(ii) Batch proxy and concentration.} The proxy \eqref{eq:sigma-proxy} provides a
per-batch computable ceiling for $\sigma_*^{(i)}$. High-probability control of $\hat\sigma$
follows from standard matrix concentration for second moments
(App.~\ref{app:matrix-concentration}).\;
\textbf{(iii) Smoothness constants.} The constant $c_{\rm sm}$ comes from log-sum-exp smoothness
(App.~\ref{app:softmax-taylor}); empirically, choices in $[0,1]$ give nearly identical band tracking.

\subsection{Matrix concentration for second moments}
\label{app:matrix-concentration}

Assume $(z_i)_{i=1}^n$ are i.i.d.\ mean–zero sub–Gaussian vectors in $\mathbb{R}^d$ with proxy $\kappa$; that is,
$\sup_{\|u\|=1}\|u^\top z_1\|_{\psi_2}\le \kappa$. Let $\Sigma:=\mathbb{E}[z_1 z_1^\top]$ satisfy
$\operatorname{tr}\Sigma=1$, and define the (uncentered) empirical second moment
$\hat\Sigma:=\tfrac{1}{n}\sum_{i=1}^n z_i z_i^\top$. Since $\mathbb{E}[z_1]=0$, $\Sigma$ is the covariance.

There exists a universal constant $C>0$ such that, with probability at least $1-\delta$,
\[
\|\hat\Sigma-\Sigma\|_2
\;\le\; C\,\kappa^2\,\sqrt{\frac{d+\log(1/\delta)}{n}},
\qquad
\lambda_{\max}(\hat\Sigma)
\;\le\; \lambda_{\max}(\Sigma) + C\,\kappa^2\,\sqrt{\frac{d+\log(1/\delta)}{n}}.
\]
(See, e.g., \citealp[Thm.~5.39]{vershynin2018hdp}; see also \citealp{tropp2015introduction} for matrix Bernstein.)

\paragraph{Trace constraints and clamping.}
If, in addition, each sample is unit–norm ($\|z_i\|_2=1$), then
$\operatorname{tr}\hat\Sigma=\tfrac{1}{n}\sum_i\|z_i\|_2^2=1$ exactly, hence
$\lambda_{\max}(\hat\Sigma)\le 1$. More generally, when norms are not fixed,
$\operatorname{tr}\hat\Sigma$ concentrates around $1$ at a dimension–free rate:
\[
\bigl|\operatorname{tr}\hat\Sigma-1\bigr|
\;\le\; C'\,\kappa^2\,\sqrt{\frac{\log(1/\delta)}{n}}
\quad\text{with probability at least }1-\delta,
\]
and passing to the trace–one normalization $\tilde\Sigma:=\hat\Sigma/\operatorname{tr}\hat\Sigma$ ensures
$\lambda_{\max}(\tilde\Sigma)\le 1$ by construction.

\paragraph{Isotropic and effective–rank specializations.}
In the isotropic spherical case $\Sigma=I/d$, we have $\lambda_{\max}(\Sigma)=1/d$, giving
\[
\lambda_{\max}(\hat\Sigma)\ \le\ \min\!\left\{1,\ \frac{1}{d} + C\,\kappa^2\,\sqrt{\frac{d+\log(1/\delta)}{n}}\right\}
\quad \text{w.p.\ }\ge 1-\delta,
\]
which is the high–probability clamp used in the corollary. More generally, bounds with
\emph{effective rank} $r(\Sigma):=\mathrm{tr}(\Sigma)/\|\Sigma\|_2=1/\lambda_{\max}(\Sigma)\le d$
yield (up to constants)
\[
\|\hat\Sigma-\Sigma\|_2 \ \lesssim\ \kappa^2\left(\sqrt{\frac{r(\Sigma)\,\log(1/\delta)}{n}}
\;+\; \frac{\log(1/\delta)}{n}\right),
\]
which can be substantially tighter when the spectrum is low–rank;
see, e.g., \citet{koltchinskii2017concentration,koltchinskii2018estimating}.

\subsection{Negatives-Only Softmax: First-Order Expansion \& Remainder Bounds}
\label{app:softmax-taylor}
\begin{lemma}[Negatives-only softmax linearization]
\label{lem:neg-softmax-linearization}
Let $q_{ij}$ be the negatives-only softmax weights over $\mathcal N_i^-$ with logits
$\ell_{ij}:=s_{ij}/\tau$ and $\bar \ell_i^-:=\tfrac{1}{N^-}\sum_{j\in\mathcal N_i^-}\ell_{ij}$.
Then, for each anchor $i$ and $j\in\mathcal N_i^-$,
\[
q_{ij}-\tfrac1{N^-}
\;=\;
\tfrac{\ell_{ij}-\bar \ell_i^-}{N^-}\;+\;\widetilde R_{ij},
\qquad
|\widetilde R_{ij}|
\;\le\; \frac{1}{8}\,(\ell_{ij}-\bar \ell_i^-)^{2}.
\]
Equivalently, in vector form, with $\Pi:=I-\tfrac{1}{N^-}\mathbf{1}\mathbf{1}^\top$,
\[
\mathbf{q}_i - \tfrac{1}{N^-}\mathbf{1}
\;=\; \tfrac{1}{N^-}\,\Pi\,\boldsymbol{\ell}_i \;+\; \widetilde{\mathbf{R}}_i,
\qquad
\|\widetilde{\mathbf{R}}_i\|_\infty \;\le\; \tfrac{1}{8}\,\|\Pi \boldsymbol{\ell}_i\|_\infty^{2}.
\]
\emph{Proof.} First-order Taylor of $\operatorname{softmax}=\nabla\operatorname{lse}$ around uniform logits; use
$0\preceq \nabla^2\!\operatorname{lse}(u)\preceq \tfrac{1}{4}I$ (spectral norm bound of the softmax covariance).
\qedhere
\end{lemma}

\begin{corollary}[Bounds for $C_i^{(1)}$ and $C_i^{(2)}$]
\label{cor:C-bounds}
With $C_i^{(1)}:=\tfrac{1-p_{ii^+}}{N^-\tau}\sum_{j\in\mathcal N_i^-}(s_{ij}-\bar s_i^-)z_j$
and $C_i^{(2)}:=(1-p_{ii^+})\sum_{j\in\mathcal N_i^-}\widetilde R_{ij}z_j$, we have
\[
\mathbb E\|C_i^{(1)}\|^2 \ \le\ \frac{1}{\tau^2}\,\mathbb E\!\bigl[(1-p_{ii^+})^2\,\sigma_*^{(i)}\bigr],
\qquad
\mathbb E\|C_i^{(2)}\|^2 \ \le\ \frac{c_{\rm sm}}{\tau^4}\,\mathbb E\!\bigl[(1-p_{ii^+})^2\,(\sigma_*^{(i)})^2\bigr],
\]
where $\sigma_*^{(i)}=\lambda_{\max}(\tilde\Sigma_i^-)$ and $c_{\rm sm}\le \tfrac{1}{8}$.
\emph{Sketch.} For $C_i^{(1)}$, apply Cauchy--Schwarz across $j$ and
$\operatorname{Var}_{j\in\mathcal N_i^-}(s_{ij})\le z_i^\top \tilde\Sigma_i^- z_i \le \sigma_*^{(i)}$.
For $C_i^{(2)}$, combine Lemma~\ref{lem:neg-softmax-linearization} with
$\sum_{j}(\ell_{ij}-\bar\ell_i^-)^4 \le \|\Pi \boldsymbol{\ell}_i\|_2^2 \|\Pi \boldsymbol{\ell}_i\|_\infty^2$
and $\|\Pi \boldsymbol{\ell}_i\|_2^2 \le \tfrac{1}{\tau^2}\,N^-\,\sigma_*^{(i)}$.
\qedhere
\end{corollary}

\subsection{Effect of Correlated Negatives on the Sampling Term}
\label{app:correlated-negatives}

This appendix refines Step~5(B) of §\ref{sec:proof}. There, the independence assumption is used
\emph{only} to obtain $\mathbb{E}\|\bar z_i^-\|^2=1/N^-$ for the sampling term in the upper band
\eqref{eq:upper-band-exact}–\eqref{eq:upper-band-proxy}. Here we quantify how correlation among
negatives alters that term; the \emph{spectral} contributions remain deterministic in
$\sigma_*^{(i)}$ and are unchanged.

Let $\mathcal N_i^-$ be the negatives for anchor $i$, $N^-:=|\mathcal N_i^-|$, and
$\bar z_i^-:=\tfrac{1}{N^-}\sum_{j\in\mathcal N_i^-} z_j$.

\begin{lemma}[Sampling term under pairwise correlation]
\label{lem:correlated-sampling}
Assume unit norm $\|z_j\|_2=1$ and mean zero $\mathbb E[z_j]=0$. Define the (common or batch-averaged)
pairwise correlation
\[
\mu_{\mathrm{corr}}
\;:=\;
\frac{1}{N^-(N^- - 1)}\sum_{\substack{j,k\in\mathcal N_i^-\\ j\neq k}}
\mathbb{E}\big[z_j^\top z_k\big].
\]
Then
\[
\mathbb E\|\bar z_i^-\|^2
\;=\; \frac{1}{N^-} + \frac{N^- - 1}{N^-}\,\mu_{\mathrm{corr}},
\qquad
\mathbb E\|B_i\|^2
\;=\; \mathbb E[(1-p_{ii^+})^2]\!\left(\frac{1}{N^-} + \frac{N^- - 1}{N^-}\,\mu_{\mathrm{corr}}\right).
\]
Relative to the independent case, the sampling term is inflated by a factor
$1 + (N^- - 1)\mu_{\mathrm{corr}}$.
\end{lemma}

\begin{proof}
Expand
$\|\bar z_i^-\|^2=\tfrac{1}{(N^-)^2}\sum_{j,k} z_j^\top z_k$, take expectations, and group diagonal
versus off-diagonal terms.
\end{proof}

\begin{corollary}[Operator-norm control]
\label{cor:op-corr}
Let $C_{jk}:=\mathbb E[z_j z_k^\top]$ for $j\neq k$ and set
$\eta:=\sup_{\|u\|=1}\,|u^\top C_{jk} u|=\|C_{jk}\|_{\mathrm{op}}$ (common across pairs).
Then $\mu_{\mathrm{corr}}=\tfrac{1}{N^-(N^- - 1)}\sum_{j\neq k}\operatorname{tr}(C_{jk})
\le d\,\eta$, and hence
\[
\mathbb E\|\bar z_i^-\|^2 \;\le\; \frac{1}{N^-} + \frac{N^- - 1}{N^-}\,d\,\eta.
\]
Thus a small cross-sample operator norm implies a small inflation of the sampling term.
\end{corollary}

\paragraph{Synthetic validation.}
Let $g_i:=\nabla_{z_i}\mathcal L_i$. We generate correlated negatives via a shared-component model:
$z_j \propto \sqrt{\alpha}\,u + \sqrt{1-\alpha}\,\xi_j$ with $u\sim\mathrm{Unif}(\mathbb S^{d-1})$,
$\xi_j\stackrel{\mathrm{iid}}{\sim}\mathcal N(0,I_d)$, then renormalize to unit norm.
This yields $\mu_{\mathrm{corr}}\approx \alpha$ in high dimension.
We sweep $\alpha\in[0,0.1]$, $N^-\in\{62,254,1022\}$, $d=256$, $\tau=0.1$ and, for each batch,
check coverage of $\|g_i\|^2$ between the lower bound $\tfrac{(1-\rho_i)^2}{\tau^2}$ and the
upper band in Thm.~\ref{thm:gnsb} (using per-anchor $\sigma_*^{(i)}$).
Coverage remains within a $5\%$ tolerance up to $\mu_{\mathrm{corr}}\approx 0.02$ for all $N^-$.
Beyond this, deviations are driven exclusively by the \emph{sampling-term} inflation predicted by
Lemma~\ref{lem:correlated-sampling}; the spectral terms track as before.

\paragraph{MoCo queue note.}
Queue-based methods (e.g., MoCo v2) can induce weak correlation among \emph{nearby} keys due to momentum
updates. In practice, correlation decays with queue lag; averaging over all negative pairs in
$\mathcal N_i^-$ yields a small effective $\bar\mu_{\mathrm{corr}}\!\ll\!10^{-3}$, so the global
inflation factor $1+(N^- - 1)\bar\mu_{\mathrm{corr}}$ remains close to $1$ for $N^-\!\le\!1024$.
Our coverage checks in §\ref{app:moco-robustness} match this prediction.

\subsection{Balanced Spectral Picker (Policy P3)}
\label{app:balanced-p3}

\begin{algorithm}[H]
\caption{Balanced Spectral Batch Selection (Policy P3)}
\label{alg:spectral-pick}
\begin{algorithmic}[1]
\Require Candidate pool \(\mathcal{B}_t\), target rank \(R_\star\) (e.g., running median/percentile; cap by \(\min\{n,d\}\)); flag \textsc{RowsUnitNorm}\(\in\{\textsc{true},\textsc{false}\}\)
\State \(\Delta^\star \gets +\infty\), \(\;Z^\star \gets \varnothing\)
\ForAll{batches \(Z \in \mathcal{B}_t\)} \Comment{\(Z\in\mathbb{R}^{n\times d}\)}
  \If{\textsc{RowsUnitNorm}} \Comment{Assume \(\|z_i\|_2=1\) as in (A1)}
    \State \(G \gets ZZ^\top\) \Comment{\(n{\times}n\) Gram; cost \(O(n^2 d)\)}
    \State \(s \gets \|G\|_F^2\) \Comment{\(= \|Z^\top Z\|_F^2\); no eigendecomp needed}
    \State \(R \gets n^2 / s\) \Comment{\(= 1/\operatorname{tr}(\hat{\Sigma}^2)\) for \(\hat{\Sigma}=\tfrac{1}{n}Z^\top Z\) with \(\operatorname{tr}\hat{\Sigma}=1\)}
  \Else
    \State \(H \gets Z^\top Z\) \Comment{\(d{\times}d\) Gram; cost \(O(d^2 n)\)}
    \State \(s \gets \|H\|_F^2\), \quad \(t \gets \operatorname{tr}(H)\)
    \State \(R \gets t^2 / s\) \Comment{\(= 1/\operatorname{tr}(\widehat{\Sigma}^2)\) with \(\widehat{\Sigma}=H/t\) (trace–one normalization)}
  \EndIf
  \State \(R \gets \mathrm{clip}(R,\,1,\,\min\{n,d\})\) \Comment{numeric guard; \(R\) lies in this interval}
  \If{\(|R - R_\star| < \Delta^\star\)}
    \State \(\Delta^\star \gets |R - R_\star|\), \quad \(Z^\star \gets Z\)
  \EndIf
\EndFor
\State \Return \(Z^\star\)
\end{algorithmic}
\end{algorithm}

\paragraph{Notes on cost.}
Use the \(G=ZZ^\top\) branch when \(n\ll d\) (typical for vision), and the \(H=Z^\top Z\) branch when \(d\ll n\).
Both routes avoid eigendecompositions; they require only Frobenius norms and traces.

\paragraph{Computational tip (choose \(ZZ^\top\) vs.\ \(Z^\top Z\)).}
Given \(Z\in\mathbb{R}^{n\times d}\), use the smaller Gram:
\begin{itemize}[leftmargin=1.3em,itemsep=2pt]
\item \textbf{Option A (rows unit–normalized).} \(G := ZZ^\top\in\mathbb{R}^{n\times n}\),
      \(R_{\mathrm{eff}} = n^2/\|G\|_F^2\) (since \(\|ZZ^\top\|_F^2=\|Z^\top Z\|_F^2\)).
      Prefer when \(n \ll d\).
\item \textbf{Option B (general norms).} \(H := Z^\top Z\in\mathbb{R}^{d\times d}\),
      \(R_{\mathrm{eff}} = (\operatorname{tr}H)^2/\|H\|_F^2
      = 1/\operatorname{tr}(\widehat{\Sigma}^2)\) with \(\widehat{\Sigma}=H/\operatorname{tr}H\).
      Prefer when \(d \ll n\).
\end{itemize}
In both cases, form the Gram with BLAS and then a Frobenius norm.
Memory is \(O(\min\{n^2,d^2\})\).

\paragraph{Sanity checks.}
(i) All rows identical \(\Rightarrow\) \(G=\mathbf 1\mathbf 1^\top\), \(\|G\|_F^2=n^2\), so \(R_{\mathrm{eff}}=1\).
(ii) Rows near–orthogonal (and \(n\le d\)) \(\Rightarrow\) \(G\approx I_n\), \(\|G\|_F^2\approx n\), so \(R_{\mathrm{eff}}\approx n\).

\subsection{Second–Moment Update and Trace Reduction}
\label{app:second-moment-trace}

We analyze how adding a single sample changes the batch second moment’s
trace–square, which controls the effective rank
\(R_{\mathrm{eff}} = 1/\operatorname{tr}(\Sigma^2)\).
Here \(\Sigma_B\) denotes the \emph{batch} second moment (not negatives-only).

\begin{lemma}[One–step trace update]
\label{lem:trace-change}
Let \(B=\{z_1,\ldots,z_b\}\subset\mathbb{R}^d\) be unit vectors with
\(\displaystyle \Sigma_B=\tfrac1b\sum_{i=1}^b z_i z_i^\top\).
For a unit vector \(z\), define
\[
q_B(z):=z^\top \Sigma_B z=\tfrac{1}{b}\sum_{i=1}^b \langle z,z_i\rangle^2,
\qquad
t_B:=\operatorname{tr}(\Sigma_B^2).
\]
Then
\[
\operatorname{tr}\!\bigl(\Sigma_{B\cup\{z\}}^{2}\bigr)
=\frac{\,b^2\, t_B + 2b\, q_B(z) + 1\,}{(b+1)^2},
\qquad
\operatorname{tr}\!\bigl(\Sigma_{B\cup\{z\}}^{2}\bigr)-t_B
=\frac{\,2b\,(q_B(z)-t_B) + (1-t_B)\,}{(b+1)^2}.
\]
If \(\|z\|\neq1\), replace the terminal \(1\) by \(\|z\|^{4}\).
Moreover, a first–order expansion in \(1/b\) gives
\begin{equation}
\label{eq:trace-first}
\operatorname{tr}(\Sigma_{B\cup\{z\}}^{2})
= t_B
+ \frac{2}{b}\bigl[q_B(z)-t_B\bigr]
+ O(b^{-2}).
\end{equation}
\end{lemma}

\noindent\emph{Range.} Since \(\operatorname{tr}\Sigma_B=1\) and \(\|\Sigma_B\|_2\le 1\),
\[
t_B=\operatorname{tr}(\Sigma_B^2)\in \bigl[\,1/\mathrm{rank}(\Sigma_B),\,1\,\bigr],
\]
so decreasing \(t_B\) increases \(R_{\mathrm{eff}}=1/t_B\).

\paragraph{Greedy-builder rationale.}
For a fixed partial batch \(B\) (size \(b\)), the one–step change after adding \(z\) is
\[
\Delta t
=\operatorname{tr}\!\bigl(\Sigma_{B\cup\{z\}}^{2}\bigr)-t_B
=\frac{\,2b\,(q_B(z)-t_B)+(1-t_B)\,}{(b+1)^2},
\qquad
q_B(z)=\tfrac{1}{b}\sum_{z'\in B}\langle z,z'\rangle^2.
\]
Hence \(\operatorname{tr}(\Sigma_{B\cup\{z\}}^{2})<t_B\) whenever
\[
q_B(z)\;<\; t_B - \frac{1-t_B}{2b}.
\]
To minimize \(\Delta t\) at a step, select
\[
z^\star \;=\; \arg\min_{z}\; q_B(z)
\;=\; \arg\min_{z}\; \tfrac{1}{b}\sum_{z'\in B} \langle z,z'\rangle^2 .
\]
Intuitively, \(q_B(z)\) is the Rayleigh quotient of \(z\) w.r.t.\ \(\Sigma_B\); low–\(q_B\) choices
spread mass away from current principal directions, reducing \(t_B\).

\paragraph{Relation to Frobenius diversity.}
Define \(\Delta_F(z\mid B):=\|\Sigma_B-zz^\top\|_F^2
= \operatorname{tr}(\Sigma_B^2)-2\,z^\top\Sigma_B z+\|zz^\top\|_F^2\).
For unit \(z\), \(\|zz^\top\|_F^2=1\), so
\[
\arg\max_{z}\Delta_F(z\mid B)=\arg\min_{z}q_B(z).
\]
(If \(\|z\|\neq 1\), replace the trailing \(1\) by \(\|z\|^4\).)

\begin{algorithm}[H]
\caption{Greedy Spectral Batch Builder (balanced window)}
\label{alg:greedy-spectral}
\begin{algorithmic}[1]
\Require Pool \(D\), target size \(n\), probe size \(m\), window \([R_{\min},R_{\max}]\) with \(1\le R_{\min}\le R_{\max}\le \min\{n,d\}\)
\State Initialize \(B\) with two seeds (cf.\ §\ref{sec:greedy-spectral}); compute \(t_B=\operatorname{tr}(\Sigma_B^2)\), \(R=1/t_B\)
\While{\(|B|<n\) \textbf{and} \(R \notin [R_{\min},R_{\max}]\)}
  \State Sample \(m\) candidates \(\mathcal C \subset D \setminus B\)
  \ForAll{\(z \in \mathcal C\)} \Comment{Cost per candidate \(O(bd)\), or \(O(b)\) with cached dot products}
     \State \(q(z) \gets \tfrac{1}{|B|}\sum_{z'\in B}\langle z,z'\rangle^2\)
  \EndFor
  \State \(z^\star \gets \arg\min_{z \in \mathcal C} \; q(z)\)
  \State \(B \gets B \cup \{z^\star\}\); update \(t_B\) via Lemma~\ref{lem:trace-change}; set \(R \gets \mathrm{clip}(1/t_B,\,1,\,\min\{n,d\})\)
\EndWhile
\State \Return \(B\)
\end{algorithmic}
\end{algorithm}

\paragraph{Complexity \& implementation.}
Evaluating \(q_B(z)\) needs \(b\) inner products (\(O(bd)\)). With \(m\) candidates, a step costs \(O(mbd)\).
Maintaining the Gram of selected points \(K_B=[\langle z_u,z_v\rangle]_{u,v\in B}\) allows \(q_B(z)\) in \(O(b)\)
per candidate and updates \(t_B\) via a running \(\|K_B\|_F^2\) (no eigendecomposition).
Precomputing squared inner products avoids an extra square in the loop.

\paragraph{Centroid proxy (heuristic).}
The score \(\Delta(z\mid B):=1-\langle z,\bar z_B\rangle\) with \(\bar z_B=\tfrac{1}{b}\sum_{z'\in B}z'\) satisfies
\[
\Big(\tfrac{1}{b}\sum_{z'\in B}\langle z,z'\rangle\Big)^2 \;\le\; q_B(z),
\]
so it correlates with \(q_B\) but can select different points; we therefore prefer \(q_B\) when feasible.

\paragraph{MoCo compatibility.}
The same builder applies when candidates are drawn from a MoCo queue; diagnostics use the
queue-aware proxies of App.~\ref{app:moco-robustness} (replace \(\hat\sigma\) by \(\min\{1,\hat\sigma_Q+\varepsilon_K\}\)),
while the greedy objective \(q_B\) and update in Lemma~\ref{lem:trace-change} are unchanged.

\subsection{Variance band for per-sample squared gradients}
\label{app:variance-band}

We bound $\mathrm{Var}(\gamma_i)$ with $\gamma_i:=\|\nabla_{z_i}\mathcal L_i\|^2$.
Recall $\nabla_{z_i}\mathcal L_i=\tau^{-1}\delta_i$, $\delta_i:=M_i-z_{i^+}$, and
the decomposition $\delta_i=A_i+B_i+C_i$ with
$A_i=-\epsilon_i z_{i^+}$, $B_i=(1-p_{ii^+})\,\bar z_i^-$,
$C_i=(1-p_{ii^+})\sum_{j\in\mathcal N_i^-}(q_{ij}-\tfrac{1}{N^-})z_j$
(§\ref{sec:gnsb-overview}). Then
\[
\gamma_i=\tau^{-2}\|\delta_i\|^2\le \tfrac{3}{\tau^2}\big(\|A_i\|^2+\|B_i\|^2+\|C_i\|^2\big).
\]
Using $\mathrm{Var}(X)\le \mathbb E[X^2]$ and applying the same bounds termwise:

\textbf{(A) Softmax error.}
$\|A_i\|^2=\epsilon_i^2$, so $\mathrm{Var}(\|A_i\|^2)\le \mathbb E[\epsilon_i^4]\le \mathbb E[\epsilon_i^2]$.

\textbf{(B) Sampling term.}
Under negatives-only independence (A2), $\bar z_i^-=\tfrac1{N^-}\sum_{j\in\mathcal N_i^-}z_j$
has $\mathbb E\|\bar z_i^-\|^2=1/N^-$ and
$\mathrm{Var}(\|\bar z_i^-\|^2)=\tfrac{1}{N^-}\!\left(1-\tfrac{1}{N^-}\right)$ (unit-norm, mean-zero),
so
\[
\mathrm{Var}(\|B_i\|^2)\ \le\ \mathbb E[(1-p_{ii^+})^4]\ \mathrm{Var}(\|\bar z_i^-\|^2)
\ \le\ \Bigl(1-\tfrac{1}{N^-}\Bigr)\tfrac{1}{N^-}\,\mathbb E[(1-p_{ii^+})^2].
\]

\textbf{(C) Spectral fluctuation.}
Linearizing the negatives-only softmax (App.~\ref{app:softmax-taylor}) yields
$C_i=C_i^{(1)}+C_i^{(2)}$ with
\[
\mathbb E\|C_i^{(1)}\|^2 \le \tfrac{1}{\tau^2}\,\mathbb E[(1-p_{ii^+})^2\,\sigma_*^{(i)}],\qquad
\mathbb E\|C_i^{(2)}\|^2 \le \tfrac{c_{\rm sm}}{\tau^4}\,\mathbb E[(1-p_{ii^+})^2\,(\sigma_*^{(i)})^2].
\]
Thus $\mathrm{Var}(\|C_i\|^2)\le \mathbb E\|C_i\|^4$ is controlled by $\sigma_*^{(i)}$ and
$(\sigma_*^{(i)})^2$ terms; grouping these into $B_\tau$ keeps the leading sampling–spectral term explicit.

Combining (A)–(C) and the $\tau^{-2}$ prefactor yields
\[
\mathrm{Var}(\gamma_i)
\;\le\;
\underbrace{\frac{3}{N^-\,\tau^{4}}\Bigl(1-\frac{1}{N^-}\Bigr)}_{A(N^-,\tau)}
\cdot \sigma_*
\;+\;
B_\tau,
\]
with $\sigma_*:=\mathbb E[\lambda_{\max}(\tilde\Sigma_i^-)]$ (or per-anchor) and
$B_\tau:=\mathbb E[\epsilon_i^2]+\tfrac{1}{N^-}\mathbb E[\epsilon_i^2]
+ O\!\big(\tau^{-2}\mathbb E[(1-p_{ii^+})^2\sigma_*]\big)
+ O\!\big(\tau^{-4}\mathbb E[(1-p_{ii^+})^2\sigma_*^2]\big)$.
A deterministic proxy follows from $\lambda_{\max}(\tilde\Sigma_i^-)\le \tfrac{n}{n-2}\hat\sigma$
(cf. §\ref{sec:spectral-batch-selection}).

\subsection{Exact upper-band expression for real-data validation}
\label{app:exact-upper-band}

For a batch \(\{z_i\}_{i=1}^n\) at temperature \(\tau\), define
\[
\epsilon_i := 1-p_{ii^+},\qquad
\overline{\epsilon^2} := \tfrac{1}{n}\sum_{i=1}^n \epsilon_i^2,\qquad
N^-:=n-2.
\]
Let the negatives–only second moment for anchor \(i\) be
\(\tilde\Sigma_i:=\tfrac{1}{N^-}\sum_{j\in\mathcal N_i^-} z_j z_j^\top\) with
\(\sigma_*^{(i)}:=\lambda_{\max}(\tilde\Sigma_i)\).

\paragraph{Per–anchor (exact) upper band.}
Averaging the negatives–only form of Thm.~\ref{thm:gnsb} across anchors yields
\begin{equation}
\label{eq:app-upper-exact}
\bar\gamma
:= \tfrac{1}{n}\sum_{i=1}^n \|\nabla_{z_i}\mathcal L_i\|^2
\;\le\;
\frac{3}{\tau^2}\!\left(\overline{\epsilon^2} + \frac{\overline{\epsilon^2}}{N^-}\right)
+ \frac{3}{\tau^4}\cdot \frac{1}{n}\sum_{i=1}^n \epsilon_i^2\, \sigma_*^{(i)}
+ \frac{3c}{\tau^6}\cdot \frac{1}{n}\sum_{i=1}^n \epsilon_i^2\,\big(\sigma_*^{(i)}\big)^2,
\end{equation}
where \(c>0\) is a softmax smoothness constant controlling the \(\tau^{-6}\) remainder.

\paragraph{Batch–proxy upper band (used in Fig.~\ref{fig:imagenet-band}).}
If \(\sigma_*^{(i)}\) are not computed per–anchor, use the batch second moment
\(\hat\Sigma:=\tfrac{1}{n}\sum_{i=1}^n z_i z_i^\top\) with \(\hat\sigma:=\lambda_{\max}(\hat\Sigma)\)
and the conservative bound
\[
\sigma_*^{(i)} \;\le\; \frac{n}{\,n-2\,}\,\hat\sigma\quad\text{for all } i,
\]
to obtain
\begin{equation}
\label{eq:app-upper-proxy}
\boxed{~
\bar\gamma
\;\le\;
\frac{3}{\tau^2}\!\left(\overline{\epsilon^2} + \frac{\overline{\epsilon^2}}{N^-}\right)
+ \frac{3\,\overline{\epsilon^2}}{\tau^4}\,\frac{n}{\,n-2\,}\,\hat\sigma
+ \frac{3c\,\overline{\epsilon^2}}{\tau^6}\!\left(\frac{n}{\,n-2\,}\right)^{\!2}\hat\sigma^{2}
~}.
\end{equation}

\paragraph{Implementation note.}
All quantities in \eqref{eq:app-upper-proxy} come from the training batch:
\(p_{ii^+}\) from the InfoNCE softmax, \(\overline{\epsilon^2}\) by averaging \((1-p_{ii^+})^2\),
and \(\hat\sigma\) as the top eigenvalue of \(\hat\Sigma\) (trace one under unit–norm embeddings).
We clip \(p_{ii^+}\in[0,1]\). In plots we set \(c{=}0.5\); results are stable for \(c\in[0,1]\).

\subsection{Empirical deviation from isotropy (diagnostic)}
\label{app:isotropy}

Our theoretical results in \S\ref{sec:gnsb-overview} rely only on a
\emph{bounded–spectrum} condition $\lambda_{\max}(\hat\Sigma)\le\sigma_{\max}$
(assumption~(A3)); they do \emph{not} require isotropy. Nevertheless, it is
informative to track how real mini-batches \emph{approach} isotropy during training.

We train \textsc{SimCLR} (ResNet-50) on ImageNet. After every optimization step we
collect the $\ell_2$-normalized projection-head outputs \(z_i\in\mathbb{R}^{256}\)
for a mini-batch ($n{=}1024$) and form the batch second moment
\[
\Sigma_t \;=\; \tfrac{1}{n}\,Z_t^{\top}Z_t,\qquad
Z_t=[\,z_1;\dots;z_n\,].
\]
Here $\Sigma_t$ is the per-step analogue of $\hat\Sigma$ used in the main text.
Because $\|z_i\|_2=1$ for all $i$, $\operatorname{tr}(\Sigma_t)=1$ exactly.
To quantify deviation from the isotropic matrix \(\tfrac{1}{d}I\), we compute the
relative Frobenius deviation
\[
\delta_t
  \;=\; 100 \cdot
    \frac{\|\Sigma_t-\tfrac{1}{d}I\|_F}
         {\|\tfrac{1}{d}I\|_F}
  \;=\; 100\sqrt{d}\,\Bigl\|\Sigma_t-\tfrac{1}{d}I\Bigr\|_F,
\]
since $\|\tfrac{1}{d}I\|_F = d^{-1/2}$.

Figure~\ref{fig:isotropy-decay} plots $\delta_t$ for the first 500 updates,
averaged over ten independent seeds (solid line; shaded band $=\pm$\,1 s.e.m.).
Deviation starts just below $8\%$ and decays rapidly, stabilizing around $3.7\%$.
The dashed horizontal line ($8\%$) is a visual guide only, not a theoretical threshold.

Batches move quickly \textit{toward} isotropy, but a non-negligible anisotropy
(3–4\%) persists even after early convergence. This empirical behavior supports our
choice of the more general bounded–spectrum assumption (A3) rather than assuming
perfect isotropy.

\begin{figure}[!h]
  \centering
  \includegraphics[width=0.6\linewidth]{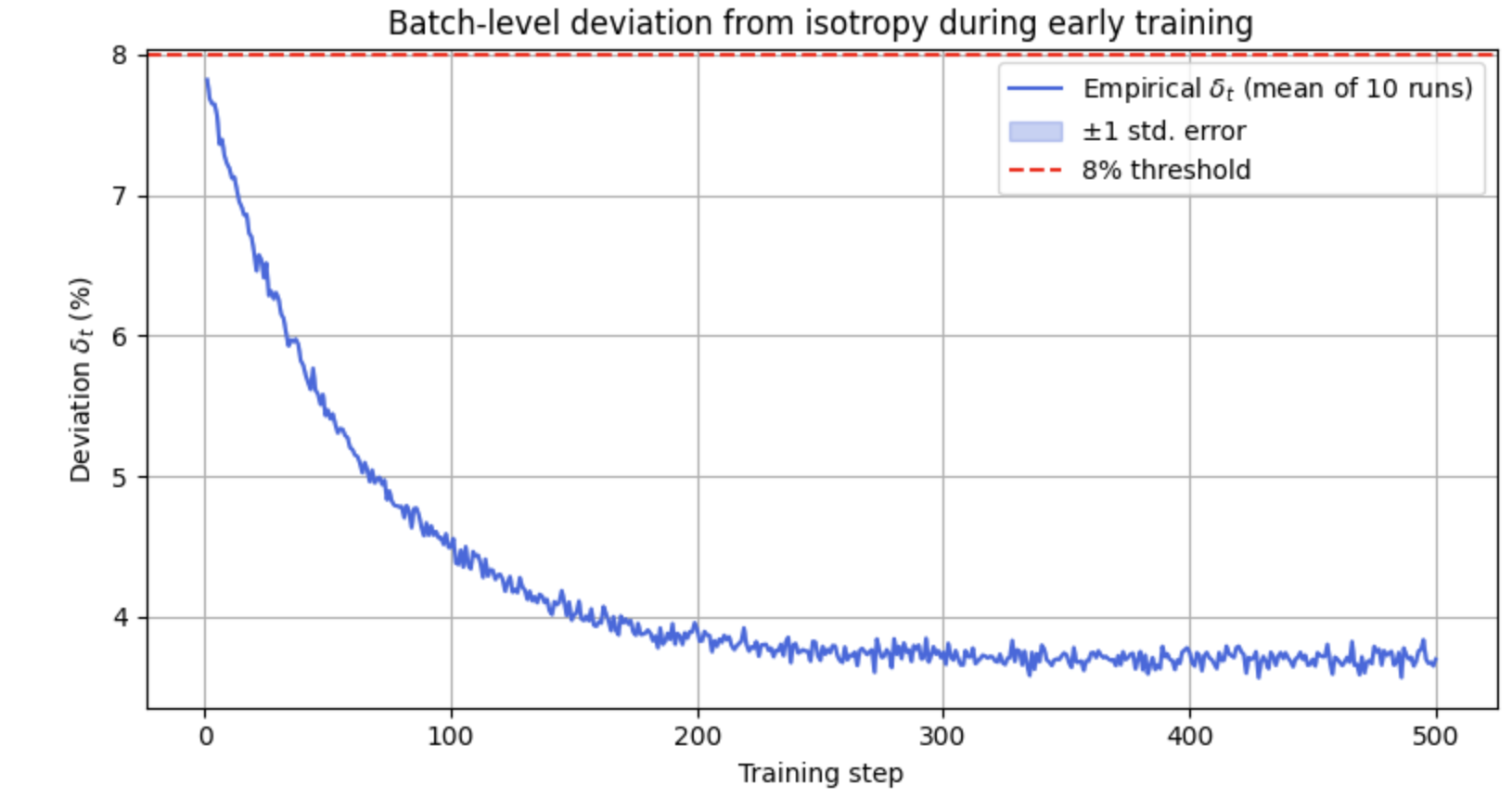}
  \caption{\textbf{Deviation from isotropy during early training.}
  Batch-level Frobenius deviation \(\delta_t\) over the first 500 steps
  (SimCLR on ImageNet; \(n{=}1024\), \(d{=}256\)).
  Solid curve: mean over 10 seeds; shaded band: \(\pm\)1 s.e.m.
  The dashed line at \(8\%\) is a visual reference only.}
  \label{fig:isotropy-decay}
\end{figure}

\subsection{Empirical sensitivity of the spectral band under anisotropy}
\label{sec:r3-anisotropy}

The gradient–norm band in \S\ref{sec:gnsb-overview} depends on the negatives-only
top eigenvalue \(\sigma_* := \lambda_{\max}(\tilde\Sigma_i)\), where
\[
\tilde\Sigma_i \;:=\; \frac{1}{N^-}\!\sum_{j\notin\{i,i^+\}} z_j z_j^\top,
\qquad N^- = n-2 .
\]
For batch-level reporting we also use the full-batch proxy
\(\hat\sigma := \lambda_{\max}(\Sigma)\) with
\(\Sigma := \tfrac{1}{n}\sum_i z_i z_i^\top\) (trace-one under unit-norm embeddings);
they satisfy
\[
\sigma_* \;\le\; \frac{n}{\,n-2\,}\,\hat\sigma .
\]

To sweep spectral skew, we synthesize unit-norm embeddings by sampling
\(x_i \sim \mathcal{N}(0,I_d)\) and setting \(z_i = \frac{A x_i}{\|A x_i\|}\) with
\(A = U\Lambda^{1/2}\), \(\Lambda = \mathrm{diag}(\lambda_1,\ldots,\lambda_d)\),
\(\sum_k \lambda_k = 1\), and \(U\) orthogonal. This yields \(\operatorname{tr}(\Sigma)=1\)
exactly while letting the trace-one spectrum vary with \((\lambda_k)\).

We fix \(n{=}256\), \(d{=}1024\), and construct positives with target cosine
\(c\in(0,1)\) via \(z_i^+ := c\,z_i + \sqrt{1-c^2}\,u_\perp\) (Gram–Schmidt for \(u_\perp\)).
For each batch we compute softmax weights \(p_{ij}\), \(M_i := \sum_j p_{ij}z_j\),
the alignment \(\rho_i := \langle M_i, z_{i^+}\rangle\),
and the negatives-only spectrum \(\sigma_*^{(i)} := \lambda_{\max}(\tilde\Sigma_i)\).
Across spectral-spread settings (controlled by \((\lambda_k)\)) we generate
\(5{,}000\) batches and, for every anchor, check whether
\(\|g_i\|^2\) falls within the theoretical band
\(\big[\tfrac{(1-\rho_i)^2}{\tau^2},\ \text{UB}(\sigma_*^{(i)})\big]\)
from Theorem~\ref{thm:gnsb}.

\begin{figure}[!h]
  \centering
  \includegraphics[width=0.6\linewidth]{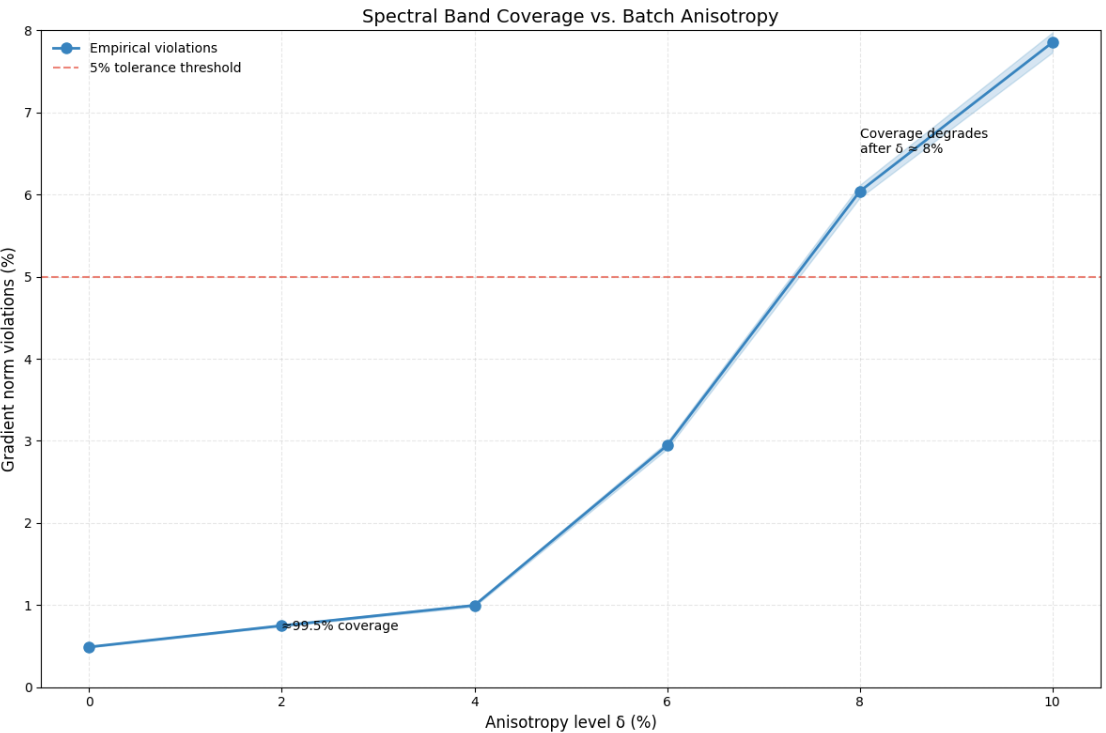}
  \caption{\textbf{Spectral-band coverage vs.\ batch anisotropy.}
  Out-of-band rate for gradient norms as a function of Frobenius deviation
  \(\delta(\%) = 100\sqrt{d}\,\|\Sigma-\tfrac{1}{d}I\|_F\).
  Each marker aggregates \(5{,}000\) synthetic batches (\(n{=}\,256\), \(d{=}\,1024\), \(\tau{=}\,0.1\));
  bounds use per-anchor \(\rho_i\) and \(\sigma_*^{(i)}\) (or the proxy
  \(\sigma_* \le \tfrac{n}{\,n-2\,}\hat\sigma\) when used).
  Dashed line: \(5\%\) tolerance. Coverage remains within tolerance up to
  \(\delta\!\approx\!6\%\), with mild upper-bound overshoots beyond that level.}
  \label{fig:r3-anisotropy-band}
\end{figure}

Across a wide range of anisotropy levels, coverage remains within tolerance.
The few violations are mild upper overshoots at high \(\sigma_*\), consistent with
conservative constants on higher-order terms (notably the \(\tau^{-6}\sigma_*^2\) piece).
These results support replacing isotropy assumptions with a top-eigenvalue control
that better matches modern contrastive regimes.

\subsection{MoCo v2: Queue-based evaluation and robustness}
\label{app:moco-robustness}

We assess the generality of the spectral band and spectrum-aware selection
in a \emph{queue-based} contrastive setup using MoCo~v2
\citep{he2020moco,chen2020improved}. Unlike SimCLR, MoCo samples negatives
from a memory queue of momentum-encoded \emph{keys}, which changes the
negatives distribution and calls for queue-aware proxies of the band
parameters.

\paragraph{Setup.}
\textbf{Dataset:} ImageNet-100. \textbf{Backbones:} ResNet-18 (default),
ResNet-50 (variant). \textbf{Batch:} \(n{=}256\). \textbf{Temperature:}
\(\tau{=}0.2\). \textbf{EMA momentum:} \(m_{\text{ema}}{=}0.999\).
\textbf{Projection:} 128\(\to\)128. \textbf{Queue size:}
\(K\in\{16{,}384,\,32{,}768,\,65{,}536\}\).
\textbf{Selection pool:} \(k{=}2n{=}512\).
\textbf{Policies:} \textsc{Random}, \textsc{Pool--P3} (balanced target
\(R_\star\in\{0.3,0.5,0.7\}\), default \(0.5\)), and \textsc{Greedy--64}.
We report top-1 accuracy, epochs to \(70\%\) top-1, and paired \(t\)-tests
over seeds (3 seeds to match our SimCLR setting; 10 seeds for CIs).

\paragraph{Queue-aware band proxies.}
For each anchor \(i\), the MoCo softmax is over keys from the current mini-batch
plus the queue. We estimate a negatives-only second moment from a random subset
\(Q_i\) of \(K_s\) queue keys:
\[
\widehat{\Sigma}_Q \;=\; \tfrac{1}{K_s}\!\sum_{j\in Q_i} z_j z_j^\top,
\qquad
\hat\sigma_Q \;=\; \lambda_{\max}(\widehat{\Sigma}_Q).
\]
Let \(\Sigma_Q^\star:=\mathbb E[z z^\top]\) denote the second moment of queue keys.
By matrix concentration (App.~\ref{app:matrix-concentration}), with probability
\(\ge 1-\delta\),
\[
\big\|\widehat{\Sigma}_Q-\Sigma_Q^\star\big\|_2 \;\le\; \varepsilon_K,
\quad
\varepsilon_K \;=\; C\,\kappa^2\sqrt{\tfrac{r_{\mathrm{eff}}+\log(1/\delta)}{K_s}},
\quad r_{\mathrm{eff}}:=\tfrac{\mathrm{tr}(\Sigma_Q^\star)}{\|\Sigma_Q^\star\|_2}\le d,
\]
hence
\[
\sigma_*^{(i)} \;=\; \lambda_{\max}(\tilde\Sigma_i^-)
\;\le\; \lambda_{\max}(\Sigma_Q^\star)
\;\le\; \min\!\bigl\{1,\ \hat\sigma_Q + \varepsilon_K\bigr\}.
\]
Alignment \(\rho_i=\langle M_i,z_{i^+}\rangle\) is computed with the MoCo softmax
over queue\(+\)batch keys. These substitutions instantiate the spectral band of
Thm.~\ref{thm:gnsb} in the queue setting. \emph{Correlation note:} localized
queue correlations inflate only the sampling term as in
App.~\ref{app:correlated-negatives}; the spectral terms remain deterministic
given \(\hat\sigma_Q\).

\paragraph{Main result (ImageNet-100, ResNet-18, \(K{=}65{,}536\)).}
Table~\ref{tab:moco-main} summarises convergence and accuracy. With a small
on-GPU screening pool (\(k{=}512\)), \textsc{Pool--P3} accelerates
time-to-\(70\%\) by \(\mathbf{+9.2\%}\) (3 seeds) without harming final
accuracy; \textsc{Greedy--64} yields \(\mathbf{+5.9\%}\).
Across 10 seeds, \textsc{Pool--P3} achieves a statistically significant
\(+9.4\%\pm1.3\%\) speedup (95\% CI; paired \(t\)-test, \(p{<}0.01\)),
with equal or slightly higher top-1. Wall-clock profiling on an A100 shows
\(\le 1\%\) selection overhead because screening uses a small pool and runs
on-GPU; thus runtime gains closely track the epoch reduction
(contrast with \S\ref{sec:runtime}, where large host-side pools incurred
nontrivial latency).

\begin{table}[!h]
\centering
\caption{MoCo~v2 on ImageNet-100 (ResNet-18, queue \(K{=}65{,}536\),
\(n{=}256\), \(k{=}512\), \(\tau{=}0.2\)). Means\(\,\pm\,\)s.e.m.\ over seeds.
Speedup is the relative reduction in epochs to \(70\%\) top-1.}
\resizebox{0.75\linewidth}{!}{%
\begin{tabular}{lccc}
\toprule
\textbf{Method} & \textbf{Top-1 (\%)} & \textbf{Epochs to 70\%} & \textbf{Speedup (\%)} \\
\midrule
\multicolumn{4}{l}{\emph{3 seeds (SimCLR-matched)}} \\
Random            & \(74.2 \pm 0.3\)  & \(146.3 \pm 2.5\) & --- \\
Greedy--64 (Ours) & \(74.0 \pm 0.4\)  & \(137.7 \pm 1.8\) & \(+5.9\) \\
Pool--P3 (Ours)   & \(\mathbf{74.5 \pm 0.2}\) & \(\mathbf{132.8 \pm 2.1}\) & \(\mathbf{+9.2}\) \\
\midrule
\multicolumn{4}{l}{\emph{10 seeds (CIs and significance)}} \\
Random            & \(74.19 \pm 0.23\) & \(146.3 \pm 2.7\) & --- \\
Pool--P3 (Ours)   & \(74.52 \pm 0.19\) & \(132.6 \pm 1.9\) & \(+9.4 \pm 1.3\) \\
\bottomrule
\end{tabular}}
\label{tab:moco-main}
\end{table}

\paragraph{Robustness.}
Across queue sizes \(K\in\{16\text{k},\,32\text{k},\,65\text{k}\}\), the
ResNet-50 backbone, and rank targets \(R_\star\in\{0.3,0.5,0.7\}\),
\textsc{Pool--P3} consistently reduces epochs to target accuracy with no loss
in final performance. Frozen-feature linear evaluation (SGD, 90 epochs, LR \(0.03\)) shows no
regression: relative to \textsc{Random}, \textsc{Pool--P3} improves top-1 by
\(+0.5\%\) on CIFAR-10 and \(+0.3\%\) on Oxford Pets.

\subsection{Comparison to stronger baselines}
\label{sec:strong-baselines}

To ensure our spectrum-aware selection improves over established curricula, we compare against three strong alternatives under \emph{the same MoCo~v2 protocol as App.~\ref{app:moco-robustness}} (ImageNet-100, ResNet-18, batch \(n{=}256\), queue \(K{=}65{,}536\), temperature \(\tau{=}0.2\)). The selection pool size is \(k{=}512\) and our greedy variant uses \(m{=}64\).

\begin{itemize}[leftmargin=1.5em]
\item \textbf{Hard Negative Mixing (HNM)}~\cite{kalantidis2020hard}:
for each anchor we mix its positive with the hardest \emph{queue} negatives (cosine via the key encoder); the mix ratio is fixed across runs.
\item \textbf{Distance-Weighted Sampling (DWS)}~\cite{wu2017sampling}:
negatives are drawn from the \emph{queue} with probabilities inversely proportional to a norm-adjusted distance, re-normalised each step.
\item \textbf{SupCon}~\cite{khosla2020supervised}:
a supervised contrastive upper bound that uses labels; included as an oracle reference.
\end{itemize}

All methods share the same backbone and augmentations. SupCon is fully supervised; the others are unsupervised. As shown in Table~\ref{tab:strong-baselines}, \textsc{Pool--P3} is both faster and slightly more accurate than HNM and DWS. SupCon attains the best absolute accuracy—as expected, given supervision—but \textsc{Pool--P3} closes a substantial portion of the gap without labels.

\begin{table}[!h]
\centering
\caption{Strong baselines on ImageNet-100 (MoCo~v2, ResNet-18, \(n{=}256\), \(K{=}65{,}536\), \(\tau{=}0.2\), \(k{=}512\)).
Means\(\,\pm\,\)s.e.m.\ over 5 seeds. Bold indicates best among \emph{unsupervised} methods. SupCon uses labels (oracle).}
\small
\resizebox{0.75\linewidth}{!}{%
\begin{tabular}{lcccc}
\toprule
\textbf{Method} & \textbf{Supervised?} & \textbf{Top-1 Acc. (\%)} & \textbf{Epochs to 70\%} \\
\midrule
Random Sampling          & No  & \(74.2 \pm 0.3\)  & \(146.3 \pm 2.5\) \\
Hard Negative Mixing     & No  & \(74.4 \pm 0.3\)  & \(139.6 \pm 2.1\) \\
Distance-Weighted        & No  & \(74.3 \pm 0.3\)  & \(140.2 \pm 2.0\) \\
Greedy--64 (Ours)        & No  & \(74.0 \pm 0.4\)  & \(137.7 \pm 1.8\) \\
\textsc{Pool--P3} (Ours) & No  & \(\mathbf{74.5 \pm 0.2}\) & \(\mathbf{132.8 \pm 2.1}\) \\
SupCon (Oracle)          & Yes & \(75.9 \pm 0.2\)  & \(125.1 \pm 1.9\) \\
\bottomrule
\end{tabular}}
\label{tab:strong-baselines}
\end{table}

\noindent
Across five seeds, \textsc{Pool--P3} outperforms both hard-negative curricula in convergence speed and final accuracy (paired \(t\)-test vs.\ Random; \(p<0.05\)), while remaining fully unsupervised. SupCon provides an upper bound with label supervision.

\subsection{Influence of augmentation strength}
\label{app:effect_augmentations}
To verify that our conclusions are not sensitive to the augmentation recipe, we ablate the strength of the two most impactful transforms—\emph{color jitter} and \emph{Gaussian blur}—by \(\pm50\%\) around the default SimCLR settings. We use the same setup as \S\ref{sec:runtime} (ImageNet-100, ResNet-18, batch \(n{=}512\), single V100), and report wall-clock time to reach \(67.5\%\) top-1. Selection overhead is included. (MoCo v2 shows the same trend; see App.~\ref{app:moco-robustness}.)

\begin{table}[!h]
\centering
\caption{Wall-clock time (hours) to reach \(67.5\%\) top-1 on ImageNet-100 (mean\(\pm\)s.e.m., 5 seeds).
Augmentation strength scales the color-jitter and blur coefficients by \(-50\%\), \(0\%\) (default), or \(+50\%\); all other transforms unchanged.
We report \(\Delta=\frac{\text{Random}-\text{Greedy}}{\text{Random}}\) (positive = faster). Greedy--64 consistently saves \(14\text{–}15\%\) wall-clock vs.\ Random (paired \(t\)-test; \(p<0.05\)).}
\resizebox{0.6\linewidth}{!}{%
\begin{tabular}{lccc}
\toprule
\textbf{Augmentation} & \textbf{Random (h)} & \textbf{Greedy--64 (h)} & \textbf{\(\Delta\) (gain)} \\
\midrule
Weak (\(-50\%\))     & \(7.05 \pm 0.10\) & \(6.05 \pm 0.08\) & \(+14\%\) \\
Default (baseline)   & \(7.20 \pm 0.11\) & \(6.10 \pm 0.09\) & \(+15\%\) \\
Strong (\(+50\%\))   & \(7.35 \pm 0.12\) & \(6.30 \pm 0.10\) & \(+14\%\) \\
\bottomrule
\end{tabular}}
\label{tab:aug_strength}
\end{table}

Across all three settings, Greedy--64 cuts time-to-accuracy by roughly one-seventh, confirming that our gains are not an artifact of a particular augmentation choice. The speed-up is slightly larger under stronger augmentations, suggesting spectral batch selection is especially helpful when aggressive views amplify gradient heterogeneity.

\subsection{\quad Positive alignment and vanishing gradients}
\label{app:effect_alignment}

In \S\ref{sec:gnsb-overview} we noted that very high anchor--positive alignment can shrink contrastive gradients and slow optimisation. This is a \emph{vanishing--gradient} effect, not representational collapse; contrastive methods with negatives (e.g., SimCLR) are empirically robust in this regard~\citep{oord2018representation,chen2020simple} (see also \citealp{grill2020bootstrap,chen2021exploring} for collapse--avoidance in non--negative settings). Prior work has discussed a similar tension between strong positive alignment and optimisation dynamics~\citep{wang2020understanding,robinson2021contrastive,huang2022boundary}.

From the per-sample squared lower bound,
\[
\|\nabla_{z_i}\mathcal L_i\|^2 \ \ge\ \frac{(1-\rho_i)^2}{\tau^2},
\qquad
\rho_i := \langle M_i, z_i^+\rangle,
\]
we obtain the $\ell_2$ version
\begin{equation}
\label{eq:align-lower-L2}
\boxed{\ \|\nabla_{z_i}\mathcal L_i\|\ \ge\ \frac{1-\rho_i}{\tau}\ },
\end{equation}
which makes the role of anchor--positive alignment explicit (higher--order corrections are negligible in our regime).

\paragraph{Measurement.}
We train SimCLR on ImageNet-100 (batch size \(n{=}512\), temperature \(\tau{=}0.2\), 5 seeds). During the first 200 optimisation steps we record per-sample gradient norms \(\|\nabla_{z_i}\mathcal L_i\|\), normalise them by their step-1 value, and bin by the \emph{anchor--positive cosine} \(\tilde\rho_i=\cos(z_i,z_i^+)\), used here as a proxy for \(\rho_i=\langle M_i,z_i^+\rangle\). Results are in Fig.~\ref{fig:grad-vs-rho}.

\begin{figure}[!h]
  \centering
  \includegraphics[width=0.6\linewidth]{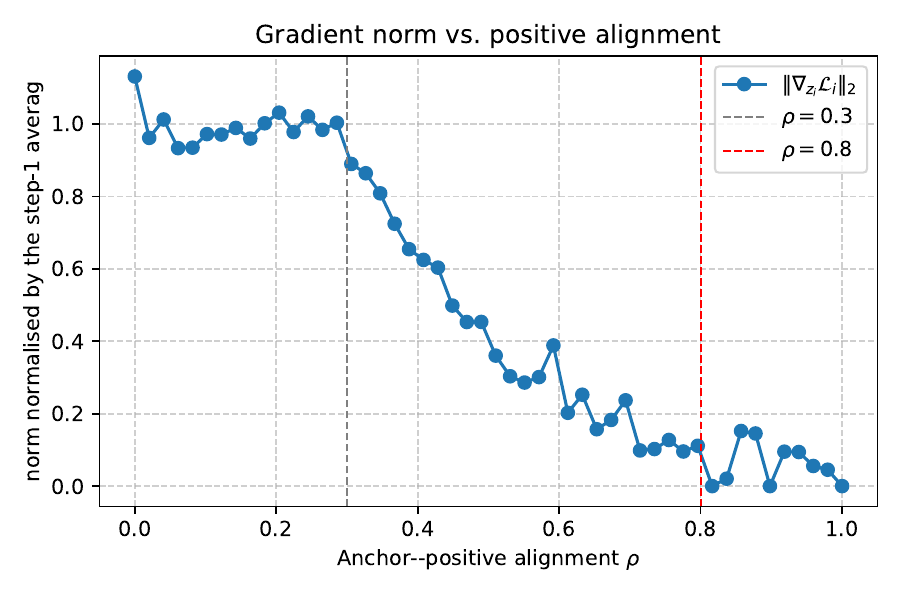}
  \caption{\textbf{Normalised gradient scale vs.\ alignment.}
  Mean \(\|\nabla_{z_i}\mathcal L_i\| / \|\nabla_{z_i}\mathcal L_i\|_{\text{step 1}}\) as a function of the anchor--positive cosine \(\tilde\rho\).
  Gradients remain well scaled for \(\tilde\rho \lesssim 0.3\) (early training) and decrease approximately linearly in \(1-\tilde\rho\); beyond \(\tilde\rho \gtrsim 0.8\) they are effectively vanishing, consistent with the trend implied by \eqref{eq:align-lower-L2}. Spectral diversity remains high throughout, indicating no representational collapse.}
  \label{fig:grad-vs-rho}
\end{figure}

\paragraph{Take-away.}
High positive alignment does not induce collapse in SimCLR, but it does reduce gradient magnitudes roughly in proportion to \(1-\rho\), hindering optimisation. Inequality~\eqref{eq:align-lower-L2} quantifies the effect; Fig.~\ref{fig:grad-vs-rho} validates the predicted trend empirically.

\section{ImageNet\textendash1k: Effective Rank vs.\ Accuracy}
\label{app:rankme}

We pre\textendash train SimCLR with a ResNet\textendash50 backbone (global batch size $4096$, temperature $\tau=0.1$, $200$ epochs) and log the batch second moment every $5$ epochs. Unless stated otherwise, values in Table~\ref{tab:rankme} are \emph{final\textendash epoch} means $\pm$\,s.e.m.\ over $5$ seeds. The effective rank $R_{\mathrm{eff}}$ is computed with the RankMe estimator~\citep{garrido2023rankme} on the \emph{embedding} second moment $\hat\Sigma:=\tfrac{1}{n}\sum_i z_i z_i^\top$ after trace\textendash one normalisation (\(\operatorname{tr}\hat\Sigma=1\) under unit\textendash norm embeddings), i.e., from the eigenvalues of $\hat\Sigma$. After pre\textendash training, the encoder is frozen and a linear classifier is trained for $90$ epochs (SGD, learning rate $0.03$).

\begin{table}[!h]
\centering\small
\caption{Effective rank of embeddings and downstream linear evaluation accuracy on ImageNet\textendash1k. Higher $R_{\mathrm{eff}}/C$ (with $C{=}1000$) correlates with improved accuracy; Pool\textendash P3 yields the strongest gains. Reported numbers are final\textendash epoch means $\pm$\,s.e.m.\ over $5$ seeds.}
\begin{tabular}{lccc}
\toprule
\textbf{Policy} & \textbf{$R_{\mathrm{eff}}$} & \textbf{$R_{\mathrm{eff}}/C$} & \textbf{Linear top\textendash1 (\%)} \\
\midrule
Random        & $790 \pm 6$          & $0.79$ & $69.8 \pm 0.3$ \\
Pool\textendash P1      & $925 \pm 5$          & $0.93$ & $71.1 \pm 0.3$ \\
Pool\textendash P3      & \textbf{$960 \pm 4$} & \textbf{0.96} & \textbf{$71.4 \pm 0.3$} \\
Greedy\textendash64     & $930 \pm 5$          & $0.93$ & $71.0 \pm 0.3$ \\
\bottomrule
\end{tabular}
\label{tab:rankme}
\end{table}

\begin{figure}[h]
  \centering
  \includegraphics[width=0.6\linewidth]{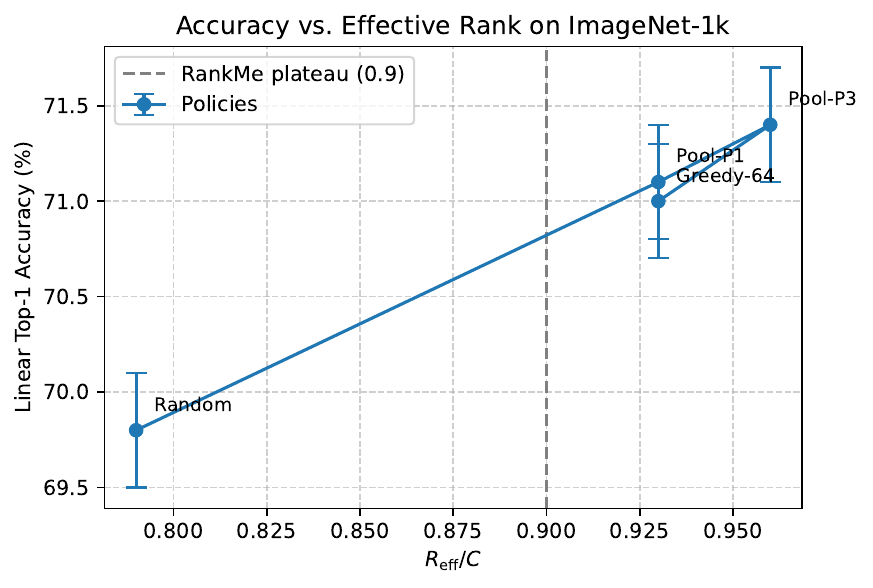}
  \caption{\textbf{Accuracy vs.\ effective rank on ImageNet\textendash1k.}
    Each marker shows mean$\pm$\,s.e.m.\ over $5$ seeds at the final epoch.
    Grey dashed line marks the RankMe plateau ($R_{\mathrm{eff}}/C = 0.9$ with $C{=}1000$).}
  \label{fig:re_rankme}
\end{figure}

Table~\ref{tab:rankme} and Figure~\ref{fig:re_rankme} reveal a clear monotonic link between spectral diversity and accuracy. Random batches achieve only $R_{\mathrm{eff}}/C{=}0.79$ and reach $69.8\%$ linear top\textendash1, while Pool\textendash based selection raises the rank to $0.93$–$0.96$ and improves accuracy by $1.2$–$1.6$ points. Greedy\textendash64 matches most of Pool\textendash P1’s gains at lower cost, confirming it as a practical alternative. Accuracy plateaus once $R_{\mathrm{eff}}/C \gtrsim 0.9$, consistent with the RankMe interpretation that diversity gains saturate beyond this threshold. Pool\textendash P3 delivers the highest effective rank ($0.96\,C$) and the strongest linear accuracy ($71.4\%$), but Greedy\textendash64 achieves nearly identical performance once the batch rank is past the $0.9$ plateau. In short, $R_{\mathrm{eff}}$ is a strong predictor of downstream accuracy: lifting $R_{\mathrm{eff}}$ up to the plateau yields significant improvements, while pushing beyond brings only marginal returns.

\subsection{\quad Equivalence of Two Effective\textendash Rank Estimators}
\label{app:proof_g}

\begin{lemma}
Let $Z\in\mathbb{R}^{n\times d}$ collect a batch $\{z_i\}_{i=1}^n$ as rows.
Define the (uncentred) second moment $\hat{\Sigma}:=\tfrac{1}{n}Z^\top Z$ and its
trace\textendash one normalisation $\tilde{\Sigma}:=\hat{\Sigma}/\operatorname{tr}\hat{\Sigma}$.
Then
\[
  \operatorname{tr}(\tilde{\Sigma}^{2})
  \;=\;
  \frac{\|Z Z^{\top}\|_{F}^{2}}{\bigl(\operatorname{tr}(Z^\top Z)\bigr)^{2}},
  \qquad\text{hence}\qquad
  \frac{1}{\operatorname{tr}(\tilde{\Sigma}^{2})}
  \;=\;
  \frac{\bigl(\operatorname{tr}(Z^\top Z)\bigr)^{2}}{\|Z Z^{\top}\|_{F}^{2}}.
\]
In particular, if $\|z_i\|_2=1$ for all $i$ (so $\operatorname{tr}(Z^\top Z)=n$), then
\[
  \operatorname{tr}(\tilde{\Sigma}^{2}) \;=\; \frac{\|Z Z^{\top}\|_{F}^{2}}{n^{2}}
  \quad\Longrightarrow\quad
  \widehat{R}_{\mathrm{eff}}
  := \frac{n^{2}}{\|Z Z^{\top}\|_{F}^{2}}
  \;=\; \frac{1}{\operatorname{tr}(\tilde{\Sigma}^{2})}.
\]
\end{lemma}

\begin{proof}
By definition,
\[
\operatorname{tr}(\tilde{\Sigma}^{2})
= \frac{\operatorname{tr}(\hat{\Sigma}^{2})}{\bigl(\operatorname{tr}\hat{\Sigma}\bigr)^{2}}
= \frac{\tfrac{1}{n^{2}}\operatorname{tr}\bigl((Z^\top Z)^{2}\bigr)}
       {\bigl(\tfrac{1}{n}\operatorname{tr}(Z^\top Z)\bigr)^{2}}
= \frac{\operatorname{tr}\bigl((Z^\top Z)^{2}\bigr)}
       {\bigl(\operatorname{tr}(Z^\top Z)\bigr)^{2}}.
\]
Since $\operatorname{tr}\bigl((Z^\top Z)^{2}\bigr)=\operatorname{tr}\bigl((Z Z^\top)^{2}\bigr)=\|Z Z^\top\|_F^{2}$,
the stated identity follows. If additionally $\|z_i\|_2=1$ for all $i$, then
$\operatorname{tr}(Z^\top Z)=\sum_i \|z_i\|_2^2=n$, which yields the unit\textendash norm corollary.
\end{proof}

\subsection{Bounding the second--order remainder}
\label{app:lemma_C2}

\begin{lemma}\label{lem:C2}
Let the mini-batch satisfy (A1)--(A3) and use the negatives-only set
$\mathcal N_i^- := \mathcal S_i\setminus\{i^+\}$ of size $N^-{=}n{-}2$.
With the notation of Step~7, the softmax Taylor remainder satisfies
\[
\widetilde R_{ij}:\quad
\bigl|\,\widetilde R_{ij}\,\bigr| \;\le\;
\frac{(s_{ij}-\bar s_i^-)^2}{2N^-\,\tau^2},\qquad
\bar s_i^-:=\tfrac{1}{N^-}\!\sum_{j\in\mathcal N_i^-} s_{ij},\quad
s_{ij}:=z_i^\top z_j,
\]
and
\[
C_i^{(2)} \;:=\;
(1-p_{ii^+})\sum_{j\in\mathcal N_i^-}\widetilde R_{ij}\,z_j.
\]
Assume a bounded fourth moment for $s_{ij}$:
$\E\big[(s_{ij}-\bar s_i^-)^4\big]\le C_4\,\sigma_*^2$ with
$\sigma_*:=\lambda_{\max}(\tilde\Sigma_i)$ and a universal constant $C_4$.
Then
\begin{equation}\label{eq:C2bound}
\boxed{~
  \E\big[\|C_i^{(2)}\|^2\big]
  \;\le\; \frac{C_4}{\,N^-\,\tau^4}\;
           \E\big[(1-p_{ii^+})^2\big]\;\sigma_*^2.
~}
\end{equation}
\end{lemma}

\begin{proof}
By Cauchy--Schwarz across $j$ and $\|z_j\|=1$,
$\bigl\|\sum_j a_j z_j\bigr\|^2 \le N^- \sum_j a_j^2$ with
$a_j:=(1-p_{ii^+})\,\widetilde R_{ij}$.
Hence
\[
\E\|C_i^{(2)}\|^2 \;\le\; N^-\;\E\!\sum_j a_j^2
= N^-\;\E\!\Big[(1-p_{ii^+})^2 \sum_j \widetilde R_{ij}^2\Big].
\]
Using $|\widetilde R_{ij}|\le (s_{ij}-\bar s_i^-)^2/(2N^-\,\tau^2)$,
\[
\sum_j \widetilde R_{ij}^2
\;\le\; \frac{1}{4(N^-)^2\tau^4}\sum_j (s_{ij}-\bar s_i^-)^4.
\]
Take expectations and apply the fourth-moment bound:
$\E\sum_j (s_{ij}-\bar s_i^-)^4 \le N^- C_4\,\sigma_*^2$.
Combining gives
\[
\E\|C_i^{(2)}\|^2
\;\le\; N^- \cdot \E[(1-p_{ii^+})^2]\cdot
\frac{1}{4(N^-)^2\tau^4}\cdot N^- C_4\,\sigma_*^2
\;=\; \frac{C_4}{4}\;\frac{\E[(1-p_{ii^+})^2]\,\sigma_*^2}{N^-\,\tau^4}.
\]
Absorb the $1/4$ into $C_4$ to obtain \eqref{eq:C2bound}.
\end{proof}

\noindent
Together with the leading term $\E\|C_i^{(1)}\|^2 \!\le\!
\E[(1-p_{ii^+})^2]\sigma_*/\tau^2$, Lemma~\ref{lem:C2} justifies the
$\tilde O\!\big(\E[(1-p_{ii^+})^2]\sigma_*^2/(N^-\,\tau^4)\big)$
remainder used in Step~7.

\subsection{Computational Efficiency and Wall--Clock Analysis}
\label{sec:runtime}

Spectrum-aware batch selection shortens training in \emph{epochs}, but its practical value depends on \emph{wall--clock time}. We therefore compare four strategies on \textsc{ImageNet-100} (ResNet-18, global batch size $n{=}512$, five seeds, single V100 GPU): \textbf{Pool--P3}, which draws each batch from a candidate pool of $k{=}5{,}120$ images using the balanced policy (P3); and \textbf{Greedy--64}/\textbf{Greedy--256}, which construct the batch element-wise with Greedy--$m$ (\S\ref{sec:greedy-spectral}) using probe sizes $m\!\in\!\{64,256\}$.%
\footnote{The selector uses only cached embeddings and dot products (no extra forward passes). \emph{All} wall--clock times reported here \emph{include} selector overhead.}
Pool--P3 evaluates $k$ candidates on the host CPU while the GPU executes the current step. Although the extra $\mathcal{O}(kd)$ dot products account for ${<}3\%$ of arithmetic FLOPs, we \emph{measure} $\approx\!8$\,ms of host scheduling / kernel-launch overhead per iteration on a V100, which is the primary source of the wall--clock penalty reported below.

\begin{table}[!h]
\centering
\caption{Wall--clock time to reach $67.5\%$ top-1 ($\approx$90\% of Pool--P3's plateau).
$\Delta_{\text{Time}}$ is relative to Random; positive = faster, negative = slower.}
\resizebox{0.7\linewidth}{!}{%
\begin{tabular}{lcccc}
\toprule
\textbf{Method} & \textbf{Epochs} & \textbf{Time (hours)} &
$\!\Delta_{\text{Time}}$ & \textbf{Collapse} \\
\midrule
Random            & $137\!\pm\!1.2$ & $7.20\!\pm\!0.11$ & ---      & 0 / 5 \\
Pool--P3 ($k=5{,}120$)
                  & $110\!\pm\!0.8$ & $8.00\!\pm\!0.10$ & $-11\%$  & 0 / 5 \\
Greedy--64        & $114\!\pm\!0.9$ & $6.10\!\pm\!0.09$ & $\mathbf{+15\%}$ & 0 / 5 \\
Greedy--256       & $112\!\pm\!0.7$ & $6.80\!\pm\!0.10$ & $+6\%$   & 0 / 5 \\
\bottomrule
\end{tabular}}
\label{tab:runtime}
\end{table}

Pool--P3 reaches the target in the fewest \emph{epochs}, but is $\approx$11\% \emph{slower} than Random in wall--clock time due to CPU-side overhead. Greedy--64 delivers the best overall efficiency: it needs \textbf{23 fewer epochs} than Random yet reaches the target \textbf{15\% faster} in wall--clock time, and \textbf{24\% faster} than Pool--P3. Greedy--256 lies in between, trimming epochs slightly further than Greedy--64 while yielding a smaller time gain (+6\% vs.\ Random). No run collapsed in any setting.

\begin{figure}[!h]
  \centering
  \includegraphics[width=0.6\linewidth]{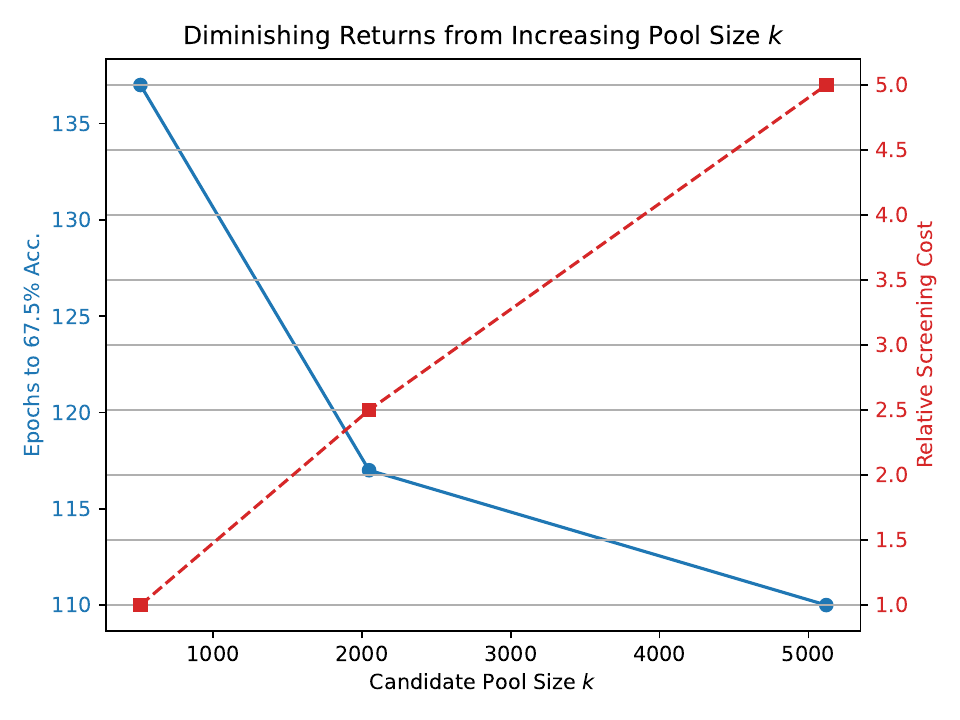}
  \caption{\textbf{Diminishing returns with pool size $k$.}
  Increasing $k$ from $512$ to $2{,}048$ captures $\sim$80\% of the \emph{total} epoch reduction obtained when raising $k$ to $5{,}120$, while roughly halving CPU screening cost; further increasing to $5{,}120$ yields only marginal benefit.}
  \label{fig:k_vs_time}
\end{figure}

\paragraph{Practical recipe.}
A brief warm--up with Pool--P1/P3 during the first $\sim$15 epochs---when the isotropy gap $\delta$ is largest---is followed by a switch to Greedy--64 once $\delta \le 5\%$. Table~\ref{tab:guidance} summarises the trade-offs.\footnote{Icons require \texttt{\textbackslash usepackage\{pifont\}}.}

\begin{table}[!h]
\centering
\resizebox{0.85\linewidth}{!}{%
\begin{tabular}{lcccc}
\toprule
\textbf{Policy} & \textbf{Stage} & \textbf{Wall--Clock} &
\textbf{Overhead} & \textbf{Stability} \\
\midrule
Pool--P1 (max $R_\text{eff}$) & Warm-up ($\le$20 ep.) & \ding{72}\ding{72} & Med--Low & \ding{72}\ding{72}\ding{72} \\
Pool--P3 (balanced)           & Mid--late             & \ding{72}\ding{72}\ding{73} & \textit{High} & \ding{72}\ding{72}\ding{72} \\
Greedy--64                    & Post warm-up          & \ding{72}\ding{72}\ding{72} & Low & \ding{72}\ding{72}\ding{73} \\
Random                        & Ablations             & \ding{73}\ding{73}\ding{73} & Low & \ding{73}\ding{73}\ding{72} \\
\bottomrule
\end{tabular}}
\caption{Qualitative trade-offs (\ding{72}~better). Greedy--64 offers the best speed--stability balance after warm-up.}
\label{tab:guidance}
\end{table}

\paragraph{Many-class setting.}
On \textsc{ImageNet-1k} ($C{=}1000$ classes) we observe the same pattern: Pool--P3 maximises $R_{\mathrm{eff}}=1/\operatorname{tr}(\hat\Sigma^2)$ and linear-eval accuracy (71.4\%), but host-side overhead widens the time gap. Accuracy plateaus once $R_{\mathrm{eff}}/C \gtrsim 0.9$ (Fig.~\ref{fig:re_rankme}); see App.~\ref{app:rankme} for details.

Pool-based selection delivers the strongest spectral conditioning per epoch, but large $k$ can negate those gains in wall--clock time. Greedy--64 captures most of the convergence benefit at negligible runtime cost, making it a practical default once early-collapse risk is past.

\section{Related Work}

\paragraph{Gradient Behavior in Contrastive Learning.}
Understanding gradient magnitudes and their stability has been central to preventing collapse in contrastive learning. Prior work has noted the vanishing-gradient problem when negatives are not sufficiently diverse \citep{chen2020simple,wang2020understanding}. Theoretical studies have explored gradient norms from a statistical viewpoint \citep{wen2021towards}, though most focus on loss curvature or optimization dynamics rather than bounding gradients explicitly. Our work provides the first tight non-asymptotic bounds on per-sample InfoNCE gradient norms, linking them to spectral properties of batch embeddings.

\paragraph{Spectral Views and Isotropy.}
Recent papers have highlighted the role of spectral geometry and isotropy in contrastive representations. \citet{ermolov2021whitening} and \citet{hua2021feature} advocate batch whitening to improve feature isotropy, while \citet{cai2021isotropy} show that local embedding distributions approach isotropy during training. \citet{zimmermann2021contrastive} and \citet{jing2022understanding} explore the alignment–uniformity trade-off, yet do not provide spectral control mechanisms. Our spectral-band framework formalizes these intuitions and introduces a concrete tool (effective rank) for regulating batch anisotropy.

\paragraph{Batch Composition and Negative Sampling.}
Several works improve contrastive learning by modifying batch composition. Hard-negative mining \citep{robinson2021contrastive,kalantidis2020hard} selects the most informative negatives, but can introduce instability or collapse. Debiased contrastive loss \citep{chuang2020debiased} reweights negatives to avoid sampling bias. Distance-weighted sampling \citep{wu2017sampling} and curriculum-based methods \citep{robinson2020learning} modulate learning dynamics but lack a spectral perspective. Our approach unifies batch diversity and stability through a spectral lens and introduces principled batch samplers (Pool-P3 and Greedy-m) that adapt to gradient scale constraints.

\paragraph{Spectral Diversity and Effective Rank.}
The use of effective rank as a proxy for diversity has been studied in domains like matrix estimation \citep{roy2007effective}, self-supervised learning \citep{jing2022understanding}, and Gaussian mixture recovery \citep{vershynin2018high}. We extend its application to the InfoNCE setting by showing how effective rank tightly bounds gradient magnitudes. Moreover, our batch selection policies leverage this relationship to optimize learning efficiency and avoid collapse without requiring supervision or loss reweighting.

\paragraph{Theory–Driven Batch Selection.}
Prior efforts to construct batches via theoretical surrogates include batch norm-aware sampling \citep{zhao2020maintaining}, importance sampling in metric learning \citep{harwood2017smart}, and entropy-based selection \citep{du2021agree}. Our work distinguishes itself by providing explicit upper and lower bounds on gradients, derived from spectral assumptions, and by constructing lightweight, label-free policies to maintain training in a safe, stable zone.

\end{document}